\documentclass{article}

\PassOptionsToPackage{numbers, compress}{natbib}
\usepackage[preprint]{nips_2018}

\usepackage[utf8]{inputenc} 
\usepackage[T1]{fontenc}    
\usepackage[bookmarks=false,pdfstartview={FitH}]{hyperref}       
\usepackage{url}            
\usepackage{booktabs}       
\usepackage{amsfonts}       
\usepackage{nicefrac}       
\usepackage{microtype}      
\usepackage{graphicx}
\usepackage{times} 
\usepackage{amsmath} 
\usepackage{amsthm} 
\usepackage{amssymb}
\usepackage{algorithmic}  
\usepackage{algorithm} 
\usepackage{epsfig} 
\usepackage{epstopdf}
\usepackage{multirow} 
\usepackage{subfigure} 
\usepackage{caption} 
\usepackage{lipsum} 
\usepackage{xcolor} 
\usepackage{bbm}
\usepackage{enumitem}

\newtheorem{definition}{Definition}
\newtheorem{theorem}{Theorem}
\newtheorem{lemma}{Lemma}[theorem]

\newtheorem{prop}{Proposition}
\newtheorem{lemcorollary}{Corollary}[lemma]

\title{Tsallis Reinforcement Learning: \\A Unified Framework for Maximum Entropy Reinforcement Learning}
\author{ 
 Kyungjae Lee$^1$, Sungyub Kim$^2$, Sungbin Lim$^3$,  Sungjoon Choi$^3$, and Songhwai Oh$^1$\\
 Dep. of Electrical and Computer Engineering, Seoul National University$^1$\\
School of Computing, KAIST$^2$\\
 Kakao Brain$^3$\\
  \texttt{kyungjae.lee@rllab.snu.ac.kr},\;\texttt{sungyub.kim@kaist.ac.kr},\\
  \texttt{\{sungbin.lim, sam.choi\}@kakaobrain.com},\\
  \texttt{songhwai@snu.ac.kr}
}
\begin{document}
\maketitle
\begin{abstract}
In this paper, we present a new class of Markov decision processes
(MDPs), called Tsallis MDPs, with Tsallis entropy maximization, which
generalizes existing maximum entropy reinforcement learning (RL).
A Tsallis MDP provides a unified framework for the original RL problem
and RL with various types of entropy, including the well-known
standard Shannon-Gibbs (SG) entropy, using an additional real-valued
parameter, called an \textit{entropic index}.
By controlling the entropic index, we can generate various types of
entropy, including the SG entropy, and a different entropy
results in a different class of the optimal policy in Tsallis MDPs.
We also provide a full mathematical analysis of Tsallis MDPs, including
the optimality condition, performance error bounds, and convergence.
Our theoretical result enables us to use any positive entropic index
in RL. 
To handle complex and large-scale problems, we propose a model-free
actor-critic RL method using Tsallis entropy maximization.  
We evaluate the regularization effect of the Tsallis entropy with
various values of entropic indices and show that the entropic index
controls the exploration tendency of the proposed method. 
For a different type of RL problems, we find that a different value
of the entropic index is desirable.
The proposed method is evaluated using the MuJoCo simulator and 
achieves the state-of-the-art performance.
\end{abstract}

\section{Introduction}

Reinforcement learning (RL) combined with a powerful function
approximation technique like a neural network has shown noticeable
successes on challenging sequential decision making problems, such as
playing a video game \cite{volodymyr15humanlevel}, learning complex
control \cite{duan2016benchmarking,gu2016modelacc}, and realistic motion generation \cite{peng2017deeploco}. 
A model-free RL algorithm aims to learn a policy that effectively performs a given task by autonomously exploring an unknown environment
where the goal of RL is to find an optimal policy which maximizes the expected sum of rewards.
Since the prior knowledge about the environment is generally not given,
an RL algorithm should consider not only gathering information about
the environment to find out which state or action leads to high
rewards but also improving its policy based on the collected
information. 
Such trade-offs should be carefully handled to reduce the sample
complexity of a learning algorithm. 

For the sake of efficient exploration of an environment, many RL
algorithms employ maximization of the Shannon-Gibbs (SG) entropy of a
policy distribution. 
It has been empirically shown that maximizing the SG entropy of a
policy along with reward maximization encourages exploration since the
entropy maximization penalizes a greedy behavior \cite{mnih2016asyncrl}.
Eysenbach et al. \cite{eysenbach2018diversity} also demonstrated that maximizing the SG entropy  
helps to learn diverse and useful behaviors.
This penalty from the SG entropy also helps to capture the multi-modal
behavior where the resulting policy is robust against unexpected
changes in the environment Haarnoja et al. \cite{haarnoja2017energy}.
Theoretically, it has been shown that the optimal
solution of maximum entropy RL has a softmax distribution of
state-action value, not a greedy policy \cite{schulman17equivalence,nachum2017bridging, haarnoja2017energy, haarnoja2018sac}. 
Furthermore, maximum SG entropy in RL provides the connection between
policy gradient and value-based learning \cite{schulman17equivalence,
  odonoghue2017pgq}.
In \cite{dai2018sbeed}, it has been also shown that maximum entropy induces a
smoothed Bellman operator and it helps stable convergence of value
function estimation. 

While the SG entropy in RL provides better exploration,
numerical stability, and capturing multiple optimal actions, it is
known that the maximum SG entropy causes a performance loss since it 
hinders exploiting the best action to maximize the reward
\cite{lee2018sparse, chow2018tsallispcl}. 
Such drawback is often handled by scheduling a coefficient of the SG
entropy to progressively vanish \cite{cesa2017boltzmann}.
However, designing a proper decaying schedule is still a demanding task
in that it often requires an additional validation step in practice.
In \cite{grau-moya2018soft}, the same issue was handled by automatically manipulating
the importance of actions using the mutual information.
On the other hand, an alternative way has been proposed to handle
the exploitation issue of the SG entropy using a sparse Tsallis
(ST) entropy \cite{lee2018sparse, chow2018tsallispcl}, which is a special case of the Tsallis entropy
\cite{tsallis1988possible}. 
The ST entropy encourages exploration while penalizing less on a
greedy policy, compared to the SG entropy. 
However, unlike the SG entropy, the ST entropy may discover a
suboptimal policy since it enforces the algorithm to explore the
environment less \cite{lee2018sparse, chow2018tsallispcl}.

In this paper, we present a unified framework for the original RL
problem and maximum entropy RL problems with various types of entropy. 
The proposed framework is formulated as a new class of Markov decision processes
with Tsallis entropy maximization, which is called Tsallis MDPs.
The Tsallis entropy generalizes the standard SG entropy and can
represent various types of entropy, including the SG and ST entropies by
controlling a parameter, called an \textit{entropic index}.
A Tsallis MDP presents a unifying view on the use of various entropies in RL.
We provide a comprehensive analysis of how a different value of the
entropic index can provide a different type of optimal policies and
different Bellman optimality equations. 
Our theoretical result allows us to interpret the effects of various
entropies for an RL problem.

We also derive two dynamic programming algorithms:
Tsallis policy iteration and Tsallis value iteration for all postive
entropic indices with optimality and convergence guarantees.
We further extend Tsallis policy iteration to a Tsallis actor-critic
method for model-free large-scale problems.
The entropic index controls the exploration-exploitation trade-off in
the proposed Tsallis MDP since a different index induces a
different bonus reward for a stochastic policy. 
Similar to the proposed method,
Chen et al. \cite{chen2018tsallisensembles} also proposed 
an ensemble network of action value combining multiple Tsallis entropies.
While ensemble network  with multiple Tsallis entropies shows efficient exploration in discrete action problems,
it has limitations in that it is not suitable for a continuous action space.

As aforementioned, it has been observed that using the SG and ST
entropies show distinct exploration tendencies. 
The former makes the policy always assign non-zero probability to all
possible actions.
On the other hand, the latter can assign zero probability to some
actions.
The proposed Tsallis RL framework contains both SG and ST entropy
maximization as special cases and allows a diverse range of
exploration-exploitation trade-off behaviors for a learning agent,
which is a highly desirable feature since the problem complexity is
different for each task at hand.
We validate our claim in MuJoCo simulations and demonstrate that the
proposed method with a proper entropic index outperforms existing
actor-critic methods. 



\section{Background}

In this section, we review a Markov decision process and define the
Tsallis entropy using \textit{q}-exponential and \textit{q}-logarithm
functions.

\subsection{Markov Decision Processes}

A Markov decision process (MDP) is defined as a tuple $\mathcal{M} = \{\mathcal{S}, \mathcal{A}, d, P, \gamma, \mathbf{r} \}$, 
where $\mathcal{S}$ is the state space, $\mathcal{F}$ is the corresponding feature space, $\mathcal{A}$ is the action space,
$d(s)$ is the distribution of an initial state,
$P(s'|s,a)$ is the transition probability from $s\in\mathcal{S}$ to $ s'\in\mathcal{S}$ by taking $a\in\mathcal{A}$,
$\gamma \in (0,1)$ is a discount factor, and $\mathbf{r}$ is the
reward function defined as $\mathbf{r}(s,a,s') \triangleq
\mathbb{E}\left[\mathbf{R}\middle|s,a,s'\right]$ with a random reward $\mathbf{R}$.
In our paper, we assume that $\mathbf{r}$ is bounded. 
Then, the MDP problem can be formulated as follows:
\begin{eqnarray}\label{def:mdp}
\small
\begin{aligned}
& \underset{\pi\in\Pi}{\text{maximize}}
& & \mathop{\mathbb{E}}_{\tau \sim P, \pi}\left[\sum_{t}^{\infty}\gamma^{t}\mathbf{R}_t \right],
\end{aligned}
\end{eqnarray}
where $\sum_{t}^{\infty}\gamma^{t}\mathbf{R}_t$ is a discounted sum of
rewards, also called a return, $\Pi$ is a set of policies, 
$\{ \pi | \forall s, a \in \mathcal{S}\times\mathcal{A}, \pi(a|s) \ge 0, \sum_{a}\pi(a|s) = 1 \}$, 
and $\tau$ is an infinite sequence of state-action pairs sampled from
the transition probability and policy, i.e., 
$s_{t+1} \sim P(\cdot|s_{t},a_{t}), a_{t} \sim \pi(\cdot|s_{t})$ for $t\in [0,\infty]$ and $s_{0} \sim d$ .
For a given $\pi$, we can define the state value and state-action (or action) value as 
\begin{eqnarray}\label{def:v_q_pi}
\small
\begin{aligned}
V^{\pi}(s) \triangleq& \mathop{\mathbb{E}}_{\tau \sim P, \pi}\left[\sum_{t=0}^{\infty}\gamma^{t}\mathbf{R}_{t} \middle| s_{0} = s \right],\\
Q^{\pi}(s,a) \triangleq& \mathop{\mathbb{E}}_{\tau \sim P, \pi}\left[\sum_{t=0}^{\infty}\gamma^{t}\mathbf{R}_{t} \middle| s_{0} = s, a_{0}=a \right].
\end{aligned}
\end{eqnarray}
The solution of (\ref{def:mdp}) is called the optimal policy
$\pi^{\star}$.
The optimal value $V^{\star} = V^{\pi^{\star}}$ and action-value
$Q^{\star} = Q^{\pi^{\star}}$ satisfy the Bellman optimality equation
as follows:  
For $\forall s,a$, 
\begin{eqnarray}\label{def:bellman_opt}
\small
\begin{aligned}
&Q^{\star}(s,a) =  \mathop{\mathbb{E}}_{s' \sim P}\left[ \mathbf{r}(s,a,s') + \gamma V^{\star}(s') \right]\\
&V^{\star}(s) =  \max_{a'} Q^{\star}(s,a'), \pi^{\star}\in\arg\max_{a'} Q^{\star}(s,a'),
\end{aligned}
\end{eqnarray}
where $\arg\max_{a'} Q^{\star}(s,a')$ indicates a set of policies satisfying $\mathop{\mathbb{E}}_{a \sim \pi}[Q^{\star}(s,a)] = \max_{a'} Q^{\star}(s,a')$
and $a \sim \pi^\star$ indicates $a \sim \pi^\star(\cdot|s)$.
Note that there may exist multiple optimal policies if the optimal action value has multiple maxima with respect to actions.

\subsection{\textit{q}-Exponential, \textit{q}-Logarithm, and Tsallis Entropy}



\begin{figure}[t!]
\vspace{-5pt}
\includegraphics[width=0.9\textwidth]{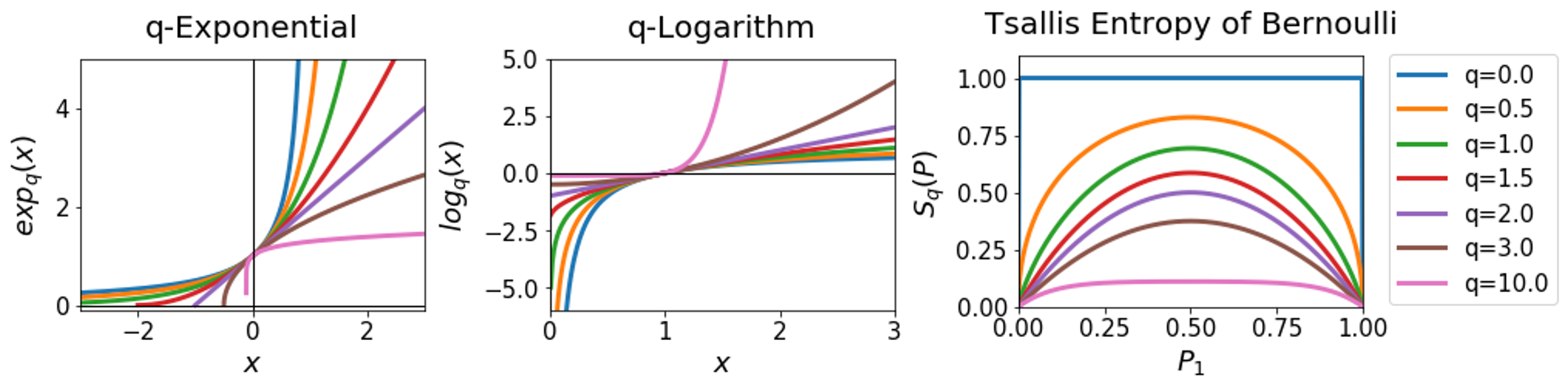}
\centering
\caption{Examples of $\exp_{q}(x),\ln_{q}(x)$, and $S_{q}(P)$ of a Bernoulli distribution parameterized by $P_{1}\triangleq P(X=1)$.} \label{fig:tsallis_ent_exm}
\vspace{-15pt}
\end{figure}

Before defining the Tsallis entropy,
let us first introduce variants of exponential and logarithm
functions, which are called \textit{q}-exponential and
\textit{q}-logarithm, respectively. 
They are used to define the Tsallis entropy 
and defined as follows\footnote{Note that the definition of $\exp_{q},\ln_{q}$, and Tsallis entropy are slightly different from the original one \cite{amari2011geometry} but 
those settings are recovered by setting $q=2-q'$.}: 
\begin{eqnarray}\label{def:qexp}
\small
\begin{aligned}
\exp_{q}(x) \triangleq
\begin{cases}
    \exp(x) ,  & \text{if } q = 1\\
    [1+(q-1)x]_{+}^{\frac{1}{q-1}}, & \text{if } q \neq 1,
\end{cases}
\end{aligned}
\end{eqnarray}
\begin{eqnarray}\label{def:qlog}
\small
\begin{aligned}
\ln_{q}(x) \triangleq
\begin{cases}
    \log(x) ,  & \text{if } q = 1 \text{ and } x > 0\\
   \frac{x^{q-1}-1}{q-1}, & \text{if } q \neq 1 \text{ and } x > 0,
\end{cases}
\end{aligned}
\end{eqnarray}
where $[x]_{+} = \max(x,0)$ and $q$ is a real number.

The property of \textit{q}-exponential and \textit{q}-logarithm
depends on the value of $q$.
We would like to note that, for all $q$, $\ln_{q}(x)$ is a monotonically increasing function with respect to $x$, since its gradient is always positive.
In particular, $\ln_{q}(x)$ becomes $\log(x)$ when $q\rightarrow1$
and a linear function when $q=2$.
However, the tendencies of their gradients are different. 
For $q > 2$, $d \ln_{q}(x) / dx$ is an increasing function.
On the contrary, for $q < 2$, $d\ln_{q}(x) /dx$ is a decreasing function.
Especially, when $q=2$, its gradient becomes a constant.
This property will play an important role for controlling the
exploration-exploitation trade-off when we model a policy using 
parameterized function approximation. 
Now, we define the Tsallis entropy using $\ln_{q}(x)$.
\begin{definition} 
The Tsallis entropy of a random variable $X$ with the distribution $P$
is defined as follows \cite{amari2011geometry}:
\begin{eqnarray}\label{def:tsallis_ent}
\small
\begin{aligned}
S_{q} (P) \triangleq \mathop{\mathbbm{E}}_{X \sim P}\left[- \ln_{q}(P(X))\right].
\end{aligned}
\end{eqnarray}
Here, $q$ is called an \textit{entropic-index}.
\end{definition}
The Tsallis entropy can represent various types of entropy by varying the \textit{entropic index}.
For example, when $q\rightarrow1$, $S_{1}(P)$ becomes the Shannon-Gibbs entropy
and when $q=2$, $S_{2}(P)$ becomes the sparse Tsallis entropy \cite{lee2018sparse}.
Furthermore, when $q\rightarrow\infty$, $S_{q}(P)$ converges to zero.
Exmaples of \textit{q}-exponential, \textit{q}-logarithm, and $S_{q}(P)$ are shown in Figure \ref{fig:tsallis_ent_exm}.
We would like to emphasize that, for $q > 0$, the Tsallis entropy is a
concave function with respect to the density function, but, for $q
\leq 0$, the Tsallis entropy is a convex function. 
Detail proofs are included in the supplementary material.
In this paper, we only consider the case when $q > 0$ since our
purpose of using the Tsallis entropy is to give a bonus reward to a
stochastic policy. 

\section{Bandit with Maximum Tsallis Entropy}

Before applying the Tsallis entropy to an MDP problem,
let us consider a simpler case, which is a stochastic multi-armed bandit (MAB) problem
defined by only a reward function and action space, i.e., $\{\mathcal{A},\mathbf{r}\}$,
where the reward only depends on an action, i.e., $\mathbf{r}(a) = \mathbbm{E}[\mathbf{R}|a]$.
While an MAB is simpler than an MDP, many properties of an MAB are
analogous to that of an MDP. 

In this section, we discuss an MAB with Tsallis entropy maximization defined as 
\begin{eqnarray}\label{prob:exp_mte}
\small
\begin{aligned}
\max_{ \pi \in \Delta } \left[ \mathop{\mathbbm{E}}_{a\sim \pi}\left[ \mathbf{R} \right] + \alpha S_{q} (\pi)\right],
\end{aligned}
\end{eqnarray}
where $\Delta$ is a probability simplex whose element is a probability
distribution and $\alpha$ is a coefficient.
We assume that an action $a$ is a discrete random variable in this
analysis but an extension to a continuous action space is discussed in
the supplementary material.  
Furthermore, we assume that $\alpha=1$ and, for $\alpha\neq 1$, by replacing $\mathbf{r}$ with $\mathbf{r}/\alpha$, the following analysis will hold.
The Tsallis entropy leads to a stochastic optimal policy and the problem (\ref{prob:exp_mte}) has the following solution
\begin{eqnarray}\label{def:optsol_qmax}
\small
\begin{aligned}
\pi^{\star}_{q}(a) = \exp_{q}\left(\mathbf{r}(a)/q- \psi_{q}\left(\mathbf{r}/q\right)\right),
\end{aligned}
\end{eqnarray}
where $\psi_{q}(\cdot)$ is called a \textit{q}-potential function
\cite{amari2011geometry}, which is uniquely determined by the normalization
condition: 
\begin{eqnarray}\label{def:qpoten}
\small
\begin{aligned}
\sum_{a} \pi^{\star}_{q}(a) = \sum_{a}\exp_{q}\left(\mathbf{r}(a)/q- \psi_{q}\left(\mathbf{r}/q\right)\right)= 1.
\end{aligned}
\end{eqnarray}
A detail derivation can be found in the supplementary material.
Note that $\psi_{q}$ is a mapping from $\mathbf{r}(\cdot)$ to a real value.
The optimal solution (\ref{def:optsol_qmax}) assigns a probability
$q$-exponentially proportional to the reward $\mathbf{r}(a)$ 
and $\psi_{q}$ normalizes the probability.
Since $\exp_{q}(x)$ is an increasing function when $q>0$,
the optimal solution assigns high probability to an action with a high
reward. 
Furthermore, using $\pi_{q}^{\star}$, the optimal value can be written as 
\begin{eqnarray}
\small
\begin{aligned}
&\mathop{\mathbbm{E}}_{a\sim \pi^{\star}}\left[ R \right] + S_{q} (\pi^{\star})=\\
&(q-1) \sum_{a} \frac{\mathbf{r}(a)}{q}\exp_{q}\left(\frac{\mathbf{r}(a)}{q}- \psi_{q}\left(\frac{\mathbf{r}}{q}\right)\right) +  \psi_{q}\left(\frac{\mathbf{r}}{q}\right).
\end{aligned}
\end{eqnarray}
Now, we can analyze how the optimal policy varies with respect to $q$.
As mentioned in the previous section, we only consider $q>0$.

When $q\rightarrow 1$, the optimal policy becomes a softmax distribution and $\psi_{q}$ becomes a log-sum-exp operator, i.e., $\pi_{1}^{\star} = \exp(\mathbf{r}(a) - \psi_{1}(\mathbf{r}))$ and $\psi_{1}(\mathbf{r}) =  \log\sum_{a}\exp(\mathbf{r}(a))$
and the optimal value becomes the same as $\psi_{1}(\mathbf{r})$.
The softmax distribution is the well-known solution of the MDP with the Shannon-Gibbs entropy \cite{ziebart2010modeling, schulman17equivalence,haarnoja2017energy,nachum2017bridging}.
When $q=2$, 
the problem (\ref{prob:exp_mte}) becomes the probability simplex
projection problem, where $\mathbf{r}$ is projected into $\Delta$.
It leads to $\pi_{2}^{\star}(a)=[1+(\mathbf{r}(a)/2-\psi_{2}(\mathbf{r}/2))]_{+}$ and $\psi_{2}(\mathbf{r}/2) =1+(\sum_{a\in S} \mathbf{r}(a)/2-1)/|S|$, where $S$ is a supporting set, i.e., $S=\{ a | \pi_{2}^{\star}(a) > 0\}$.
$\pi_{2}^{\star}(a)$ is the optimal solution of the MDP with the ST entropy \cite{lee2018sparse,chow2018tsallispcl}.
Compared to $\pi_{1}^{\star}$, $\pi^{\star}_{2}$ allows zero probability to the action whose $\mathbf{r}(a)$ is below $2\psi_{2}(\mathbf{r}/2) - 2$, whereas $\pi^{\star}_1$ cannot.
Furthermore, when $q\rightarrow\infty$,
the problem becomes the original MAB problem since $S_{q}$ becomes zero as $q$ goes to infinity.
In this case, the optimal policy only assigns positive probability to optimal actions.
If $\mathbf{r}(a)$ has a single maximum, then the optimal policy becomes greedy.

Unfortunately, finding a closed form of $\psi_{q}$ for a general value
of $q$ is intractable except $q=1,2,\infty$ since it is the sum
of radical equations with the index $q$. 
Thus, for other values of $q$, the solution can be obtained using a
numerical optimization method. 
This intractability of (\ref{def:qpoten}) hampers the viability of the
Tsallis entropy.
Chen et al. \cite{chen2018tsallisensembles} handled this issue by obtaining an
approximated closed form using the first order Tayler expansion of
$\psi_{q}$. 
However, we propose an alternative way to avoid numerical
computation, which will be discussed in Section \ref{sec:tac}. 

From aforementioned observations, we can observe that, as $q$
increases from zero to infinity, the optimal policy becomes more
sparse and finally converges to a greedy policy. 
The optimal policy with different $q$ is shown in Figure
\ref{fig:tsallis_bandit}. 
The effect of different $\alpha$ and $q$ can be found in the
supplementary material. 
This tendency of the optimal policy also appears in the MDP with the
Tsallis entropy. 
Many existing methods employ the SG entropy to encourage the exploration.
However, the Tsallis entropy allows us to select the proper entropy
according to the property of the environment.  

\subsection{$q$-Maximum}

Before extending from MAB to MDP, we define the problem
(\ref{prob:exp_mte}) as an operator, which is called
\textit{q}-maximum.   
A \textit{q}-maximum operator is a bounded approximation of the maximum operator.
For a function $f(x)$, \textit{q}-maximum is defined as follows:
\begin{eqnarray}\label{def:qmax}
\small
\begin{aligned}
\mathop{q\text{-max}}_{x}\left(f(x)\right) \triangleq \max_{ P \in \Delta } \left[ \mathop{\mathbbm{E}}_{X\sim P}\left[ f(X) \right] + S_{q} (P)\right],
\end{aligned}
\end{eqnarray}
where $\Delta$ is a probability simplex whose element is a probability.
The reason why this operator (\ref{def:qmax}) is called a \textit{q}-maximum
is that it has the following bounds.
\begin{theorem} \label{thm:bound}
For any function $f(x)$ defined on a finite input space $\mathcal{X}$,
the \textit{q}-maximum satisfies the following inequalities.
\begin{eqnarray}\label{def:qmax_bnd}
\small
\begin{aligned}
\mathop{q\textnormal{-max}}_{x}\left(f(x)\right) &+ \ln_{q} \left(1/|\mathcal{X}|\right) \leq \max_{x}(f(x)) \leq \mathop{q\textnormal{-max}}_{x}\left(f(x)\right)
\end{aligned}
\end{eqnarray}
where $|\mathcal{X}|$ is the cardinality of $\mathcal{X}$.
\end{theorem}
The proof can be found in the supplementary material.
The proof of Theorem \ref{thm:bound} utilizes the definition of \textit{q}-maximum. 
This bounded property will be used to analyze the performance bound of an MDP with the maximum Tsallis entropy.
Furthermore, $q$-maximum plays an important role in the optimality condition of Tsallis MDPs.

\begin{figure}[t!]
\vspace{-5pt}
\includegraphics[width=0.9\textwidth]{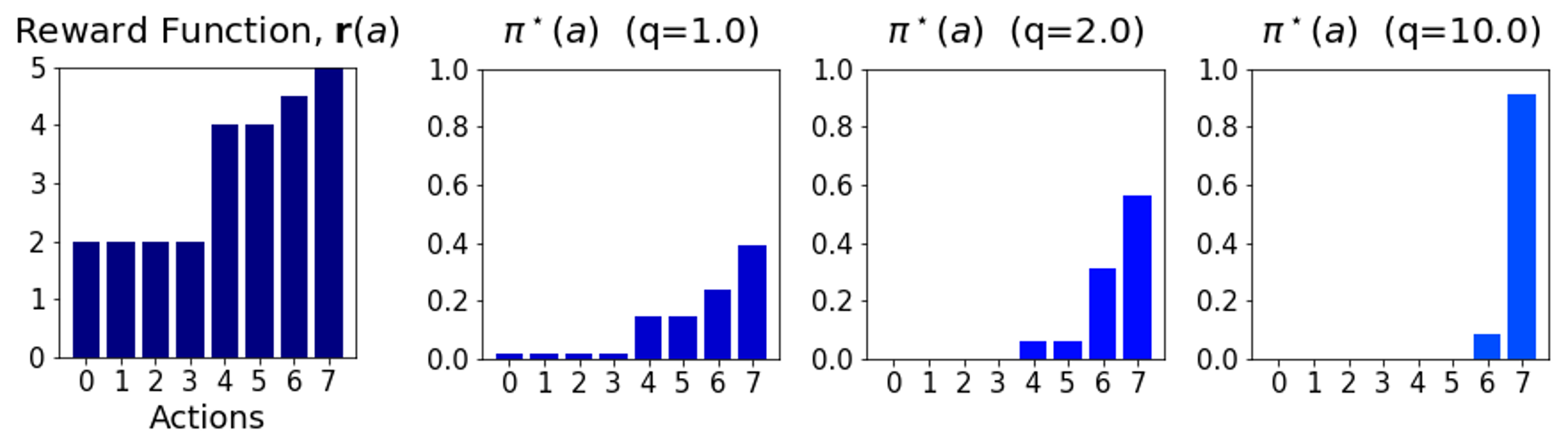}
\centering
\caption{Examples of $\pi_{q}^{\star}$ with different $q$ values. The first figure shows a given reward function over eight actions. }\label{fig:tsallis_bandit}
\vspace{-15pt}
\end{figure}

\section{Maximum Tsallis Entropy in MDPs}
In this section,
we formulate MDPs with Tsallis entropy maximization, which will be named Tsallis MDPs, by extending the SG entropy to the Tsallis entropy.
We mainly focus on deriving the optimality conditions and algorithms
generalized for the entropic index so that a wide range of $q$ values
can be used for a learning agent.
First, we extend the definition of the Tsallis entropy so that it can
be applicable for a policy distribution in MDPs. 
The Tsallis entropy of a policy distribution $\pi$ is defined by
\begin{equation*}
\small
\begin{aligned}
S_{q}^{\infty} (\pi) \triangleq\mathop{\mathbb{E}}_{\tau \sim P, \pi}\left[ \sum_{t=0}^{\infty} \gamma^{t} S_{q} ( \pi( \cdot | s_t ) ) \right].
\end{aligned}
\end{equation*}
Using $S_{q}^{\infty}$, the original MDPs can be converted into
Tsallis MDPs by adding $S_{q}^{\infty} (\pi)$ to the objective
function as follows:
\begin{eqnarray}
\small
\begin{aligned}\label{def:tsallis_mdp}
	& \underset{\pi\in\Pi}{\text{maximize}}
	& &  \mathop{\mathbb{E}}_{\tau \sim P, \pi}\left[\sum_{t}^{\infty}\gamma^{t}\mathbf{R}_t \right]+ \alpha S_{q}^{\infty} (\pi),
\end{aligned}
\end{eqnarray}
where $\alpha > 0$ is a coefficient of the Tsallis entropy.
A state value and state-action value are redefined for Tsallis MDPs as follows:
\begin{eqnarray}\label{def:tsallis_v_q_pi}
\small
\begin{aligned}
V^{\pi}_{q}(s) &\triangleq \mathop{\mathbb{E}}_{\tau \sim P, \pi}\left[\sum_{t=0}^{\infty}\gamma^{t}\left(\mathbf{R}_{t} + \alpha S_{q} ( \pi( \cdot | s_t )\right)\middle| s_{0} = s \right],\\
Q^{\pi}_{q}(s,a) &\triangleq \mathop{\mathbb{E}}_{\tau \sim P, \pi}\left[\mathbf{R}_{0} +\gamma V^{\pi}_{q}(s_{1})\middle| s_{0} = s, a_{0}=a \right],
\end{aligned}
\end{eqnarray}
where $q$ is the entropic index.
The goal of a Tsallis MDP is to find an optimal policy distribution
which maximizes both the sum of rewards and the Tsallis entropy whose
importance is determined by $\alpha$. 
The solution of the problem (\ref{def:tsallis_mdp}) is denoted as $\pi^{\star}_{q}$
and its value functions are denoted as $V_{q}^{\star}=V_{q}^{\pi^{\star}_{q}}$ and $Q_{q}^{\star}=Q_{q}^{\pi^{\star}_{q}}$, respectively.
In our analysis, $\alpha$ is set to one, however one can easiliy generalize the case of $\alpha\neq1$ by replacing $\mathbf{r}, V$, and $Q$ with $\mathbf{r}/\alpha,V/\alpha$, and $Q/\alpha$.

In the following sections, we first derive the optimality condition of
(\ref{def:tsallis_mdp}), which will be called the Tsallis-Bellman
optimality (TBO) equation. 
Second, dynamic programming to solve Tsallis MDPs is proposed with
convergence and optimality guarantees. 
Finally, we provide the performance error bound of the optimal policy
of the Tsallis MDP, where the error is caused by the Tsallis entropy
term. 
The theoretical results derived in this section are extended to a
viable actor-critic algorithm in Section \ref{sec:tac}. 

\subsection{Tsallis Bellman Optimality Equation}

Using the $q$-maximum operator, the optimality condition of a Tsallis
MDP can be obtained as follows.

\begin{theorem}\label{thm:tbo}
For $q>0$, an optimal policy $\pi^{\star}_{q}$ and optimal value $V^{\star}_{q}$ sufficiently and necessarily satisfy the following Tsallis-Bellman optimality (TBO) equations:
\begin{eqnarray}
\small
\begin{aligned}\label{eqn:tbo}
Q^{\star}_{q}(s,a) &=\mathop{\mathbbm{E}}_{s'\sim P}[ \mathbf{r}(s,a,s') + \gamma V^{\star}_{q}(s')|s,a] \\
V^{\star}_{q}(s) &= \mathop{q\textnormal{-max}}_{a}(Q^{\star}_{q}(s,a))\\
\pi^{\star}_{q}(a|s)  &= \exp_{q}\left(Q^{\star}_{q}(s,a)/q - \psi_{q}\left(Q^{\star}_{q}(s,\cdot)/q\right)\right),
\end{aligned}
\end{eqnarray}
where $\psi_q$ is a q-potential function.
\end{theorem}

\begin{proof}[Proof Sketch]
Unlike to the original Bellman equation,
we derive Theorem \ref{thm:tbo} from Karush-Kuhn-Tucker (KKT) conditions instead of the Bellman's principle of optimality \cite{puterman2014markov}.
The proof consists of three steps. 
First, the optimization variable in (\ref{def:tsallis_mdp}) is converted to a state-action visitation $\rho$ based on \cite{syed2008allinearprogram, puterman2014markov}.\footnote{$\rho(s,a) \triangleq \mathop{\mathbb{E}}_{\tau \sim P, \pi}\left[\sum_{t}^{\infty}\gamma^{t}\mathbbm{1}(s_{t}=s,a_{t}=a) \right]$, where $\mathbbm{1}(\cdot)$ is an indicator function.}
Second, after changing variables,
(\ref{def:tsallis_mdp}) becomes a concave problem with respect to $\rho$.
Thus, we can apply KKT conditions since the strong duality holds.
Finally, TBO equations are obtained by solving the KKT conditions.
The entire proof of Theorem \ref{thm:tbo} is included in the supplementary material.
\end{proof}

The TBO equation differs from the original Bellman equation in that
the maximum operator is replaced by the $q$-maximum operator. 
The optimal state value $V^{\star}_{q}$ is the $q$-maximum of the
optimal state-action value $Q^{\star}_{q}$ 
and the optimal policy $\pi^{\star}_{q}$ is the solution of $q$-maximum (\ref{def:qmax}).
Thus, as $q$ changes, $\pi^{\star}_{q}$ can represent various types of
$q$-exponential distributions. 
We would like to emphasize that the TBO equation becomes the original Bellman equation as $q$ diverges into infinity.
This is reasonable tendency since, as $q\rightarrow\infty$,
$S_{\infty}$ tends zero and the Tsallis MDP becomes the original MDP.
Furthermore, when $q\rightarrow1$, $q$-maximum becomes the
log-sum-exponential operator and the Bellman equation of maximum SG
entropy RL, (a.k.a. soft Bellman equation) \cite{haarnoja2017energy}
is recovered. 
When $q=2$, the Bellman equation of maximum ST entropy RL,
(a.k.a. sparse Bellman equation) \cite{lee2018sparse} is also
recovered. 
Moreover, our result guarantees that the TBO equation holds for every
real value $q>0$. 

\section{Dynamic Programming for Tsallis MDPs}

In this section, 
we develop dynamic programming algorithms for a Tsallis MDP: 
Tsallis policy iteration (TPI) and Tsallis value iteration (TVI).
These algorithms can compute an optimal value and policy and their
convergence can be shown.
TPI is a policy iteration method which consists of policy evaluation
and policy improvement.
In TPI, first, a value function of a fixed policy is computed and, then, the policy is updated using the value function.
TVI is a value iteration method which computes the optimal value directly.
In dynamic programming of the original MDPs, the convergence is derived from the maximum operator.
Similarly, in the MDP with the SG entropy, log-sum-exponential plays a
crucial role for the convergence.
In TPI and TVI, we generalize such maximum or log-sum-exponential
operators by the $q$-maximum operator, which is a more abstract notion
and available for all $q>0$. 

\subsection{Tsallis Policy Iteration}

We first discuss the policy evaluation method in a Tsallis MDP, which computes $V_{q}^{\pi}$ and $Q_{q}^{\pi}$ for fixed policy $\pi$.
Similar to the original MDP,
a value function of a Tsallis MDP can be computed using the expectation equation defined by
\begin{eqnarray}
\small
\begin{aligned}\label{eqn:tbe}
Q^{\pi}_{q}(s,a) &= \mathop{\mathbbm{E}}_{s'\sim P}[\mathbf{r}(s,a,s') + \gamma V^{\pi}_{q}(s')|s,a] \\
V^{\pi}_{q}(s) &= \mathop{\mathbbm{E}}_{a\sim \pi}[Q^{\pi}_{q}(s,a) - \ln_{q}(\pi(a|s))],
\end{aligned}
\end{eqnarray}
where $s'\sim P$ indicates $s'\sim P(\cdot|s,a)$  and $a\sim\pi$ indicates $a\sim\pi(\cdot|s)$.
Equation (\ref{eqn:tbe}) will be called the Tsallis Bellman
expectation (TBE) equation and it is derived from the definition of
$V_{q}^{\pi}$ and $Q_{q}^{\pi}$. 
Based on the TBE equation, we can define the operator for an arbitrary function $F(s,a)$ over $\mathcal{S}\times\mathcal{A}$,
which is called the TBE operator,
\begin{eqnarray}
\small
\begin{aligned}\label{def:tbe_op_q}
\left[\mathcal{T}_{q}^{\pi}F \right](s,a) &\triangleq \mathop{\mathbbm{E}}_{s' \sim P}[ \mathbf{r}(s,a,s') + \gamma V_{F}(s') |s,a]\\
V_{F}(s) &\triangleq \mathop{\mathbbm{E}}_{a\sim \pi}[F(s,a) - \ln_{q}(\pi(a|s))].
\end{aligned}
\end{eqnarray}
Then, the policy evaluation method for a Tsallis MDP can be simply defined as 
repeatedly applying the TBE operator to an initial function $F_{0}$: $F_{k+1} = \mathcal{T}_{q}^{\pi}F_k$.

\begin{theorem}[Tsallis Policy Evaluation]\label{thm:tpe}
For fixed $\pi$ and $q>0$, consider the TBE operator $\mathcal{T}_{q}^{\pi}$,
and define Tsallis policy evaluation as
$F_{k+1}=\mathcal{T}_{q}^{\pi}F_{k}$ for an arbitrary initial function
$F_{0}$ over $\mathcal{S}\times\mathcal{A}$. Then, $F_{k}$ converges
to $Q_{q}^{\pi}$ and satisfies the TBE equation (\ref{eqn:tbe}).
\end{theorem}
The proof of Theorem \ref{thm:tpe} relies on the contraction
property of $\mathcal{T}_{q}^{\pi}$ and the proof can be found in the
supplementary material. 
The contraction property guarantees the sequence of $F_{k}$ converges
to a fixed point $F_{*}$ of $\mathcal{T}_{q}^{\pi}$, i.e.,
$F_{*}=\mathcal{T}_{q}^{\pi}F_{*}$ and the fixed point $F_{*}$ is the
same as $Q^{\pi}_{q}$. 

The value function evaluated from Tsallis policy evaluation
can be employed to update the policy distribution.
In the policy improvement step,
the policy is updated to maximize 
\begin{eqnarray}
\small
\begin{aligned}\label{def:tp_imp}
\forall s, \, \pi_{k+1}(\cdot|s) =&\\
\arg\max_{\pi(\cdot|s)} &\mathop{\mathbbm{E}}_{a\sim \pi}[ Q^{\pi_{k}}_{q}(s,a) - \ln_{q}(\pi(a|s))|s]
\end{aligned}
\end{eqnarray}
\begin{theorem}[Tsallis Policy Improvement]\label{thm:tpi}
For $q>0$, let $\pi_{k+1}$ be the updated policy from (\ref{def:tp_imp}) using $Q_{q}^{\pi_{k}}$.
For all $(s,a)\in\mathcal{S}\times\mathcal{A}$,
$Q^{\pi_{k+1}}_{q}(s,a)$ is greater than or equal to $Q^{\pi_{k}}_{q}(s,a)$.
\end{theorem}
Theorem \ref{thm:tpi} tells us that the policy obtained by the
maximization (\ref{def:tp_imp}) has performance no worse than the
previous policy.
From Theorem \ref{thm:tpe} and \ref{thm:tpi},
it is guaranteed that the Tsallis policy iteration gradually improves
its policy as the number of iterations increases and it converges to
the optimal solution.

\begin{theorem}[Optimality of TPI]\label{thm:opt_tpi}
When $q>0$, define the Tsallis policy iteration as alternatively applying (\ref{def:tbe_op_q}) and (\ref{def:tp_imp}), then $\pi_{k}$ converges to the optimal policy.
\end{theorem}
The proof is done by checking if the converged policy satisfies the TBO equation.
In the next section, Tsallis policy iteration is extended to a Tsallis
actor-critic method which is a practical algorithm to handle
continuous state and action spaces and complex environments. 

\subsection{Tsallis Value Iteration}

Tsallis value iteration is derived from the optimality condition.
From the TBO equation, the TBO operator is defined by
\begin{eqnarray}
\small
\begin{aligned}\label{def:tbo_op}
\left[\mathcal{T}_{q}F \right](s,a) &\triangleq \mathop{\mathbbm{E}}_{s'\sim P}\left[\mathbf{r}(s,a,s') + \gamma V_{F}(s)\middle|s,a\right]\\
V_{F}(s) &\triangleq \mathop{q\text{-max}}_{a'}\left(F(s,a')\right).
\end{aligned}
\end{eqnarray}
Then, Tsallis value iteration (TVI) is defined by repeatedly applying
the TBO operator: $F_{k+1} = \mathcal{T}_{q}F_{k}$.

\begin{theorem}\label{thm:optimality}
For $q>0$, consider the TBO operator $\mathcal{T}_{q}$,
and define Tsallis value iteration as $F_{k+1}=\mathcal{T}_{q}F_{k}$ for an arbitrary initial function $F_{0}$ over $\mathcal{S}\times\mathcal{A}$. Then, $F_{k}$ converges to $Q_{q}^{\star}$.
\end{theorem}
Similar to Tsallis policy evaluation,
the convergence of Tsallis value iteration depends on the
contraction property of $\mathcal{T}_{q}$, which makes $F_{k}$
converges to a fixed point of $\mathcal{T}_{q}$. 
Then, the fixed point can be shown to satisfy the TBO equation.

%

\subsection{Performance Error Bounds}

\begin{figure}[t!]
\vspace{-12pt}
\centering
\subfigure[World Model]{\includegraphics[width=0.32\textwidth]{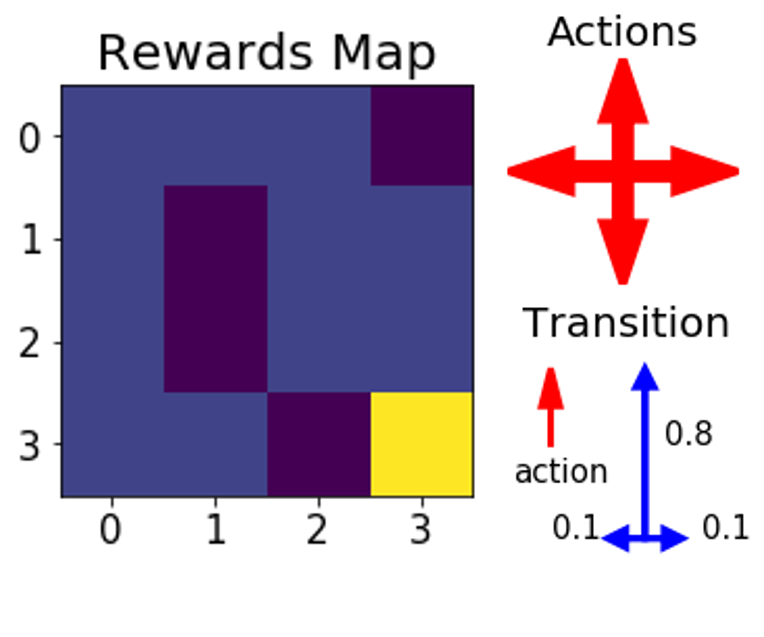}}
\subfigure[Results of Theorem \ref{thm:error_bounds}]{\includegraphics[width=0.4\textwidth]{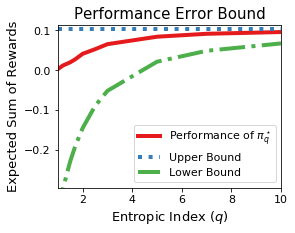}}
\caption{
(a) A reward map over $2$D grid state space and actions with a transition model. 
A dark blue (resp., blue and yellow) cell has a reward of $-0.5$ (resp., $0$ and $2$). 
(b) Performance of the optimal policy of Tsallis MDPs with varying $q$ from $1.0$ to $10.0$.}\label{fig:error_bounds}
\vspace{-15pt}
\end{figure}
We provide the performance error bounds of the optimal policy of a Tsallis MDP which can be obtained by TPI or TVI.
The error is caused by the regularization term used in Tsallis entropy maximization.
We compare the performance between the optimal policy of a Tsallis MDP and that of the original MDP.
The performance error bounds are derived as follows.

\begin{theorem}\label{thm:error_bounds}
Let $J(\pi)$ be the expected sum of rewards of a given policy $\pi$, $\pi^{\star}$ be the optimal policy of an original MDP, and $\pi^{\star}_{q}$ be the optimal policy of a Tsallis MDP with an entropic index \textit{q}.
Then, the following inequality holds:
\begin{eqnarray}
\small
\begin{aligned}
J(\pi^{\star}) + (1-\gamma)^{-1}\ln_{q} \left(1/|\mathcal{A}|\right) \leq J(\pi_{q}^{\star}) \leq J(\pi^{\star}),
\end{aligned}
\end{eqnarray}
where $|\mathcal{A}|$ is the cardinality of $\mathcal{A}$ and $q>0$.
\end{theorem}

The proof of Theorem \ref{thm:error_bounds} is included in the supplementary material.
Here, we can observe that the performance gap shows the similar
property of the TBO equation.
We further verify Theorem \ref{thm:error_bounds} on a simple grid world problem.
We compute the expected sum of rewards of $\pi_{q}^{\star}$ obtained from TVI by varying $q$ values and compare them to the bounds in Theorem \ref{thm:error_bounds}, as shown in Figure \ref{fig:error_bounds}.
Notice that $\ln_{q} \left(1/|\mathcal{A}|\right)\propto 1/|\mathcal{A}|^{q-1}$ converges to zero as $q\rightarrow\infty$.
This fact supports that $\pi_{q}^{\star}$ converges to the greedy optimal policy in the original Bellman equation when $q\rightarrow\infty$.

To summarize this section, we derive dynamic programming methods for
Tsallis MDPs with proofs of convergence, optimality, and performance error bounds.
One important result is that all theorems derived in this section hold
for every $q>0$, for which $S_{q}$ is concave.
Furthermore, the previous algorithms
\cite{ziebart2010modeling,lee2018sparse} can be recovered by setting
$q$ to a specific value. 
This generalization makes it possible to apply the Tsallis entropy to
any sequential decision making problem by choosing the proper
\textit{entropic index} depending on the nature of the problem.

TPI and TVI methods require the transition probability $P$ to
update the value function.
Furthermore, due to the intractability of $\psi_{q}$,
it also requires an additional numerical computation to evaluate $q$-maximum.
In this regard,
we extend TPI to an actor-critic method which can avoid these issues and
handle large-scale model-free RL problems. 

\section{Tsallis Actor Critic for Model-Free RL}\label{sec:tac}

In this section, we propose a Tsallis actor-critic (TAC), which can be
applied to a complex environment with continuous state and action
spaces without knowing the transition probabilities. 
To address a large-scale problem with continuous state and action
spaces, we approximate Tsallis policy iteration (TPI) using a
neural network to estimate both value and policy functions. 
In the dynamic programming setting, 
Tsallis policy improvement (\ref{def:tp_imp}) and Tsallis value iteration (\ref{def:tbo_op}) (TVI) require the same numerical computation since a closed form solution of $q$-maximum (and (\ref{def:tp_imp})) is generally not known.
However, in TAC,
the Tsallis policy improvement step is replaced by updating a policy network.

Our algorithm maintains five networks to model a policy $\pi_{\phi}$, state value $V_{\psi}$, target value $V_{\psi^{-}}$, two action values $Q_{\theta_{1}}$ and $Q_{\theta_{2}}$.
We also utilize a replay buffer $\mathcal{D}$ where every interactions $(s_{t},a_{t},r_{t+1},s_{t+1})$ are stored and it is sampled when updating the networks.
Value networks $V_{\psi}, Q_{\theta_{1}}$ and $Q_{\theta_{2}}$ are updated using (\ref{def:tbe_op_q}) and $\pi_{\phi}$ is updated using (\ref{def:tp_imp}).
Since (\ref{def:tp_imp}) has the expectation over $\pi_{\phi}$ which
is intractable, a stochastic gradient of (\ref{def:tp_imp}) is
required to update $\pi_{\phi}$. 
We employ the reparameterization trick to approximate the stochastic gradient.
In our implementation, we model a policy function as a tangent hyperbolic of a Gaussian random variable which has been first introduced in \cite{haarnoja2018sac} , i.e., $a\triangleq f_{\phi}(s;\epsilon)=\tanh(\mu_{\phi}(s)+\epsilon\sigma_{\phi}(s)), \epsilon\sim\mathcal{N}(0,\mathbf{I})$, where $\mu_{\phi}(s)$ and $\sigma_{\phi}(s)$ are the outputs of $\pi_{\phi}$.
Then, the gradient of (\ref{def:tp_imp}) becomes
$\mathop{\mathbbm{E}}_{s_{t} \sim \mathcal{D}}\left[\mathop{\mathbbm{E}}_{\epsilon \sim \mathcal{N}} \left[\nabla_{\phi}Q_{\theta}(s_{t},f_{\phi})- \alpha \nabla_{\phi}\ln_{q}(\pi_{\phi}(f_{\phi}|s_{t}))\right]\right]$,
where $f_{\phi}=f_{\phi}(s_{t},\epsilon)$ , $\mathcal{D}$ indicates a
replay buffer and $\alpha$ is a coefficient of the Tsallis entropy.
Thus, the gradient of $\ln_{q}(x)$ plays an important role in exploring the environment.
Finally, the $\psi^{-}$ is updated towards $\psi$ using an exponential moving average method.
Algorithmic details are similar to the soft actor-critic (SAC) algorithm which is known to be the state of the art.
Since we generalize the fundamental Bellman equation to the maximum Tsallis entropy case,
the Tsallis entropy can be applied to existing RL methods with the SG etropy by replacing the entropy term.
Due to the space limitation, more detailed settings are explained in the supplementary material
where the implementation of TAC are also included and it is available publicly\footnote{\url{https://github.com/kyungjaelee/tsallis_actor_critic_mujoco}}.

\section{Experiment}

In experiment, 
we verify the effect of the entropic index on exploration and compare our algorithm to the existing state-of-the-art actor-critic methods
on continuous control problems using the MuJoCo simulators:
HalfCheetah-v2, Ant-v2, Pusher-v2, Humanoid-v2, Hopper-v2, and Swimmer-v2.
Note that results for Hopper-v2 and Swimmer-v2 are included in the supplementary material.
We first evaluate how a different entropic index influences the exploration of TAC. 
As the entropic index changes the structure of the Tsallis entropy,
different entropic indices cause different types of exploration.
We also compare our method to various on-policy and off-policy actor-critic methods.
For on-policy methods, 
trust region policy optimization (TRPO) \cite{schulman2015trpo}, which
slowly updates a policy network within the trust region to obtain
numerical stability, and proximal policy optimization (PPO)
\cite{schulman2017ppo}, which utilizes an importance ratio clipping
for stable learning, are compared
where a value network is employed for generalized advantage estimation \cite{schulman2015gae}.
For off-policy methods,
deep deterministic policy gradient (DDPG) \cite{lillicrap2015ddpg},
whose policy is modeled as a deterministic function instead of
a stochastic policy and is updated using the deterministic policy gradient,
and twin delayed deep deterministic policy gradient (TD3) \cite{fujimoto2018td3}, which modifies the DDPG method by applying two Q networks to increase stability and prevent overestimation, are compared.
We also compare the soft actor-critic (SAC) method \cite{haarnoja2018sac} which employs the Shannon-Gibbs entropy for exploration.
Since TAC can be reduced to SAC with $q=1$ and algorithmic details are
the same,
we denote TAC with $q=1$ as SAC.
For other algorithms, we utilize OpenAI's implementations\footnote{\url{https://github.com/openai/spinningup}}
and extend the SAC algorithm to TAC by replacing the SG entropy with the Tsallis entropy with the entropic index $q$.
We exclude \cite{chen2018tsallisensembles}, which also can utilize the Tsallis entropy in the Q learning method,
since \cite{chen2018tsallisensembles} is only applicable for discrete action spaces.

\subsection{Effect of Entropic Index}

To verify the effect of the entropic index,
we conduct experiments with wide range of $q$ values: $\{0.5,0.7,1.0, 1.2, 1.5, 1.7, 2.0, 3.0, 5.0\}$
and measure the total average returns during training phase.
We only change the entropic index and fix an entropy coefficient $\alpha$ to $0.05$ for Humanoid-v2 and $0.2$ for other problems.
We run entire algorithms with ten different random seeds
and the results are shown in Figure \ref{fig:mujoco_tac}.
We realize that the proposed method performs better when $1 \leq q < 2$ than when $0<q<1$ and $q \geq 2$, in terms of stable convergence and final total average returns.
Using $S_{0<q<1}$ generally shows poor performance since it hinders exploitation more strongly than the SG entropy.
For $1 \leq q <2$, the Tsallis entropy penalizes less the greediness of a
policy compared to the SG entropy (or $q=1$)
where, for the same probability distribution, the value of $S_{1\leq q<2}$
is always less than $S_{1}$. 
However, the approximated stochastic gradient follows the gradient of $-\ln_{q}(x)$, which is inversely proportional to the
policy probability similar to the SG entropy.
Thus, the Tsallis entropy within $1\leq q<2$ not only encourages
exploration but also allows the policy to converge a greedy policy.
However, when $q \geq 2$, the value of the Tsallis entropy is smaller
than that of the SG entropy for a given distribution and 
the gradient of $q$-logarithm is proportional to the policy probability.
Then, the action with smaller probability is less encouraged to be
explored since its gradient is smaller than the action with larger
probability. 
Consequently, the Tsallis MDP shows an early convergence.
In this regards, we can see TAC with $1\leq q<2$ outperforms TAC with $q\geq 2$.
Furthermore, 
in HalfCheetah-v2 and Ant-v2,
TAC with $1.5$ shows the best performance in $1\leq q <2$ while, in Humanoid-v2, TAC with $1.2$ shows  the best performance.
Furthermore, in Pusher-v2, the final total average returns of all settings are similar, but TAC with $1.2$ shows slightly faster convergence.
We believe that these results empirically show that there exists the most
appropriate $q$ value between one and two depending on the environment
while $q>2$ has a negative effect on exploration.

\subsection{Comparative Evaluation}
\begin{figure*}[t!]
\vspace{-10pt}
\centering
\subfigure[HalfCheetah-v2]{\includegraphics[width=0.24\textwidth]{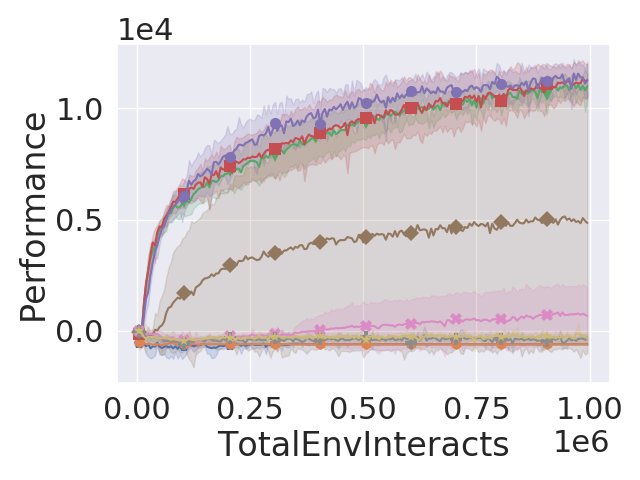}}
\subfigure[Ant-v2]{\includegraphics[width=0.24\textwidth]{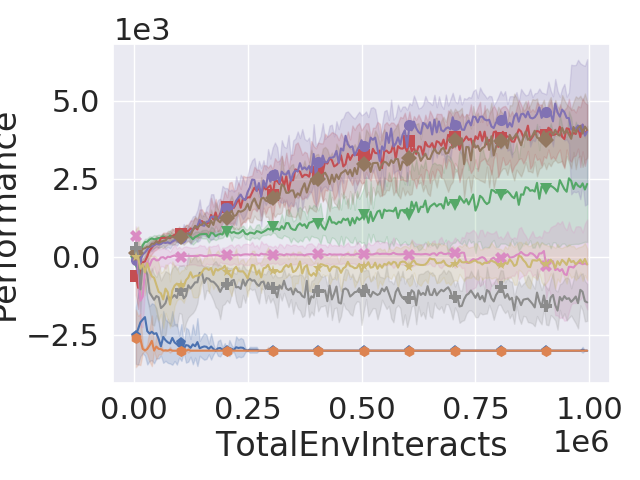}}
\subfigure[Pusher-v2]{\includegraphics[width=0.24\textwidth]{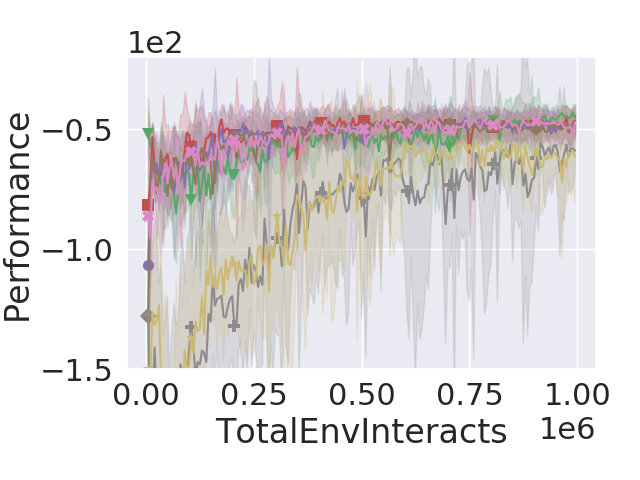}}
\subfigure[Humanoid-v2]{\includegraphics[width=0.24\textwidth]{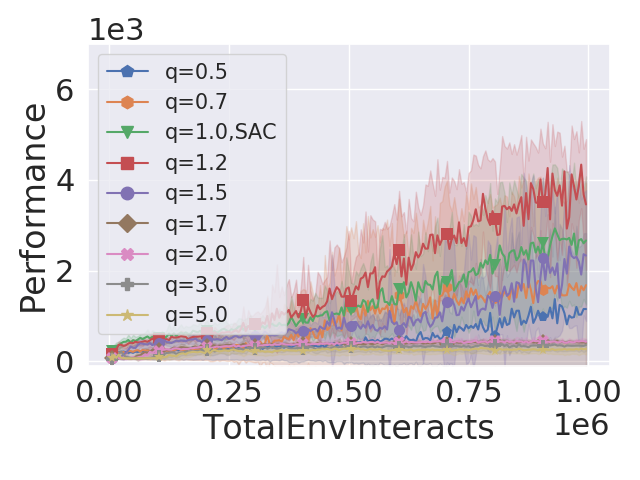}}
\caption{Average training returns on four MuJoCo tasks. 
A solid line is the average return over ten trials and the shade area
shows one variance.}
\label{fig:mujoco_tac}
\vspace{-12pt}
\end{figure*}

\begin{figure*}[t!]
\vspace{-3pt}
\centering
\subfigure[HalfCheetah-v2]{\includegraphics[width=0.24\textwidth]{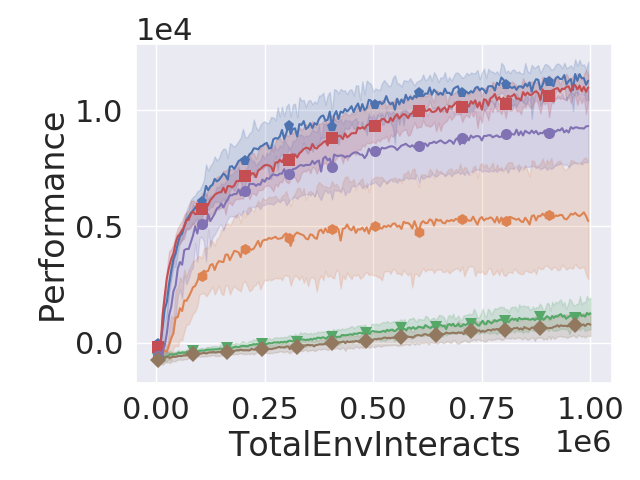}}
\subfigure[Ant-v2]{\includegraphics[width=0.24\textwidth]{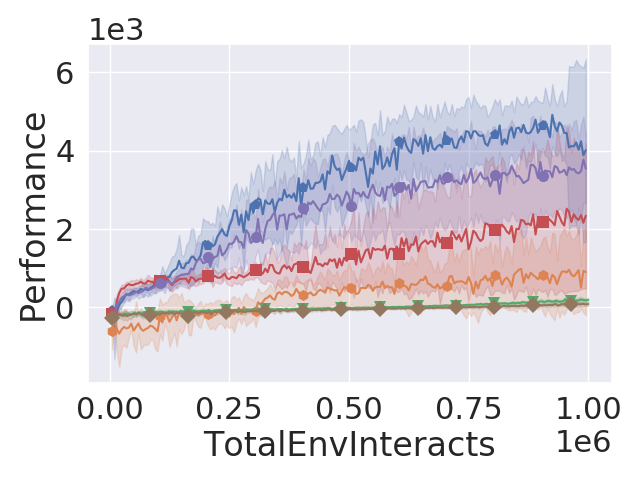}}
\subfigure[Pusher-v2]{\includegraphics[width=0.24\textwidth]{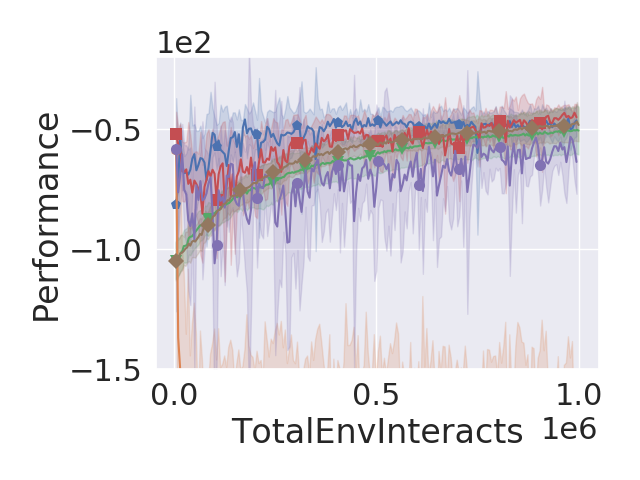}}
\subfigure[Humanoid-v2]{\includegraphics[width=0.24\textwidth]{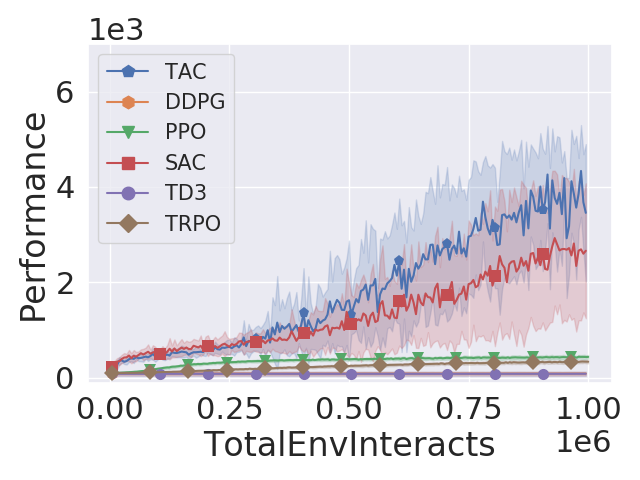}}
\caption{Comparison to existing actor-critic methods on four MuJoCo tasks. 
SAC (red square line) is the same as TAC with $q=1$ and TAC (blue pentagon line) indicates TAC with $q\neq 1$.}\label{fig:mujoco_all}
\vspace{-12pt}
\end{figure*}

Figure \ref{fig:mujoco_all} shows the total average returns of TAC and
other compared methods.
We use the best $q$ value from the previous experiments.
For SAC, the best hyperparameter reported in \cite{haarnoja2018sac} is used.
To verify that there exists a more suitable entropic index than $q=1$,
we set all hyperparameters in TAC to be the same as that of SAC and
only change $q$ values. 
SAC, TAC, and DDPG use the same architectures for actor and critic networks, which are single layered fully connected neural networks with $300$ hidden units.
However, TRPO and PPO employ a smaller architecture using $64$ units
since they shows poor performances when a large network is used.
Entire experimental settings are explained in the supplementary material.
To obtain consistent results,
we run all algorithms with ten different random seeds.
TAC with a proper $q$ value outperforms all existing methods in all environments.
While SAC generally outperforms DDPG, PPO, TRPO and shows similar or
better performance than TD3, except Ant-v2,
TAC achieves better performance with a smaller number of samples than SAC in all problems.
Especially, in Ant-v2, 
TAC achieves the best performance with $q=1.5$,
while SAC is the third best.
Furthermore, in Humnoid-v2 which has the largest action space among all problems,
TAC with $q=1.2$ outperforms all the other methods dramatically.
Although hyperparameters for TAC are not optimized, simply changing
$q$ values achieves a significant performance improvement.
These results demonstrate that, by properly setting $q$ value, TAC can achieve the state-of-the-art performance.

\section{Conclusion}

We have proposed a unified framework for maximum entropy RL problems. 
The proposed maximum Tsallis entropy MDP generalizes a MDP with
Shannon-Gibbs entropy maximization and allows a diverse range of
different entropies.
The optimality condition of a Tsallis MDP is shown
and the convergence and optimality guarantees for the proposed
dynamic programming for Tsallis MDPs have been
derived.
We have also presented the Tsallis actor-critic (TAC) method, which
can handle a continuous state action space for model-free RL
problems. 
It has been observed that there exists a suitable entropic index for a
different RL problem and TAC with a specific entropic index
outperforms all compared actor-critic methods. 
One valuable extension to this work is to learn a proper entropic
index for a given task, which is our future work.
\appendix

\section{Tsallis Entropy}

We show that the Tsallis entropy is a concave function over the distribution $P$
and has the maximum at an uniform distribution.
Note that this is an well known fact, but, we restate it to make the manuscript self-contained. 
\begin{prop}
Assume that $\mathcal{X}$ is a finite space.
Let $P$ is a probability distribution over $\mathcal{X}$.
If $q > 0$, then, $S_{q}(P)$ is concave with respect to $P$.
\end{prop}
\begin{proof}
Let us consider the function $f(x) = -x \ln_{q}(x)$ defined over $( x > 0 )$.
Second derivative of $d^{2}f(x)/dx^{2}$ is computed as 
\begin{eqnarray*}
\begin{aligned}
\frac{d^{2}f(x)}{dx^{2}} = -qx^{q-2} < 0\;\; ( x > 0, q>0 ).
\end{aligned}
\end{eqnarray*}
Thus, $f(x)$ is a concave function. 
Now, using this fact, we show that the following inequality holds.
For $\lambda_{1}, \lambda_{2}\geq0$ such that $\lambda_{1}+\lambda_{2}=1$, and probabilities $P_{1}$ and $P_{2}$,
\begin{eqnarray*}
\begin{aligned}
S_{q}(\lambda_{1}P_{1}+\lambda_{2}P_{2})
&=\sum_{x} -(\lambda_{1}P_{1}(x)+\lambda_{2}P_{2}(x))\ln_{q}(\lambda_{1}P_{1}(x)+\lambda_{2}P_{2}(x))\\
&<\sum_{x} -\lambda_{1}P_{1}(x)\ln_{q}(P_{1}(x))-\lambda_{2}P_{2}(x)\ln_{q}(P_{2}(x))\\
&=\lambda_{1}S_{q}(P_{1}) + \lambda_{2} S_{q}(P_{2}).
\end{aligned}
\end{eqnarray*}
Consequently, $S_{q}(P)$ is concave with respect to $P$.
\end{proof}
\begin{prop}
Assume that $\mathcal{X}$ is finite space.
Then, $S_{q}(P)$ is maximized when $P$ is a uniform distribution, i.e., $P=1/|\mathcal{X}|$ where $|\mathcal{X}|$ is the number of elements in $\mathcal{X}$.
\end{prop}
\begin{proof}
We would like to employ the KKT condition on the following optimization problem.
\begin{eqnarray}\label{prob:mte}
\begin{aligned}
\max_{ P \in \Delta } S_{q} (P)
\end{aligned}
\end{eqnarray}
where $\Delta = \{ P | P(x) \geq 0, \sum_{x} P(x) = 1\}$ is a probability simplex.
Since $\mathcal{X}$ is finite,
the optimization variables are probability mass defined over each element.
The KKT condition of \ref{prob:mte} is
\begin{eqnarray*}
\begin{aligned}
\forall x \in \mathcal{X}&, \frac{\partial \left(S_{q} (\pi) - \sum_{x} \lambda^{\star}(x)P(x) - \mu^{\star}\left(1-\sum_{x}P(x)\right)\right)}{\partial P(x)}\bigg|_{P(x)=P^{\star}(x)}\\
&=-\ln_{q}(P^{\star}(x))-(P^{\star}(x))^{q-1}-\lambda^{\star}(x) + \mu^{\star}\\
&=-q\ln_{q}(P^{\star}(x))-1-\lambda^{\star}(x) + \mu^{\star} = 0\\
\forall x \in \mathcal{X}&, 0=1-\sum_{x} P^{\star}(x), P^{\star}(x) \geq 0\\
\forall x \in \mathcal{X}&, \lambda^{\star}(x) \leq 0\\
\forall x \in \mathcal{X}&, \lambda^{\star}(x)P^{\star}(x) = 0
\end{aligned}
\end{eqnarray*}
where $\lambda^{\star}$ and $\mu^{\star}$ are the Lagrangian multipliers for constraints in $\Delta$.
First, let us consider $P^{\star}(x) > 0$.
Then, $\lambda^{\star}(x)=0$ from the last condition (complementary slackness).
The first condition implies
$$
P^{\star}(x) = \exp_{q}\left(\frac{\mu^{\star}-1}{q}\right).
$$
Hence, $P^{\star}(x)$ has constant probability mass which means $P^{\star}(x)=1/|S|$ where $S=\{x|P^{\star}(x) > 0\}$ .
The optimal value is $S_{q}(P^{\star}) = - \ln_{q}(1/|S|)$.
Since $-\ln_{q}(x)$ is a monotonically decreasing function,
$|S|$ should be the largest number as possible as it can be.
Hence, $S=\mathcal{X}$ and $P^{\star}(x) = 1/|\mathcal{X}|$.
\end{proof}

\section{Bandit with Maximum Tsallis Entropy}

We first analyze a Bandit setting with maximum Tsallis entropy, which is defined as
\begin{eqnarray}\label{prob:exp_mte}
\begin{aligned}
\max_{ \pi \in \Delta } \left[ \mathop{\mathbbm{E}}_{a\sim \pi}\left[ R \right] + S_{q} (\pi)\right]
\end{aligned}
\end{eqnarray}
The following propositions are already done in several works \citet{}.
However, we restate it to introduce the concept of maximum Tsallis entropy in an expectation maximization problem.

\begin{prop}\label{rem:opt_sol_bndt}
The optimal solution of (\ref{prob:exp_mte}) is
\begin{eqnarray}\label{def:optsol_qmax}
\begin{aligned}
\pi^{\star}_{q}(a) = \exp_{q}\left(\frac{\mathbf{r}(a)}{q}- \psi_{q}\left(\frac{\mathbf{r}}{q}\right)\right),
\end{aligned}
\end{eqnarray}
where the \textit{q}-potential function $\psi_{q}$ is a functional defined on $\{\mathcal{A},\mathbf{r}\}$. $\psi_{q}$ is determined uniquely for given $\{\mathcal{A},\mathbf{r}\}$ by the following normalization condition:
\begin{eqnarray}\label{def:qpoten}
\begin{aligned}
\sum_{a} \pi^{\star}_{q}(a) = \sum_{a}\exp_{q}\left(\frac{\mathbf{r}(a)}{q}- \psi_{q}\left(\frac{\mathbf{r}}{q}\right)\right)= 1.
\end{aligned}
\end{eqnarray}
Furthermore, using $\pi_{q}^{\star}$, the optimal value can be written as 
\begin{eqnarray}\label{eqn:qmax_qpot}
\begin{aligned}
\mathop{\mathbbm{E}}_{a\sim \pi^{\star}}\left[ R \right] + S_{q} (\pi^{\star})=(q-1) \sum_{a} \frac{\mathbf{r}(a)}{q}\exp_{q}\left(\frac{\mathbf{r}(a)}{q} - \psi_{q}\left(\frac{\mathbf{r}}{q}\right)\right) +  \psi_{q}\left(\frac{\mathbf{r}}{q}\right).
\end{aligned}
\end{eqnarray}
\end{prop}
\begin{proof}
It is easy to check $\psi_{q}$ exists uniquely for given $\{\mathcal{A}, \mathbf{r}\}$. Indeed, because $\exp_{q}\in [0,\infty)$ is a continuous monotonic function, for any $\{\mathcal{A},\mathbf{r}\}$, $\sum_{a}\exp_{q}\left(\frac{\mathbf{r}}{q}-\xi\right)$ converge to $0$ and $\infty$ if $\xi$ goes to $+\infty$ and $-\infty$, respectively. Therefore by the intermediate value theorem, there exists a unique constant $\xi^{*}\in\mathbb{R}$ such that $\sum_{a} \exp_{q}\left(\frac{\mathbf{r}(a)}{q}-\xi^{*}\right) = 1$. Hence it is sufficient to take $\psi_{q}(\mathbf{r}/q) = \xi^{*}$.

To show the remains, we mainly employ the convex optimization technique.
Since $S_{q} (\pi)$ is concave and the expectation and constraints of $\Delta$ are linear w. r. t. $\pi$,
the problem is concave. Thus, strong duality holds and we can use KKT conditions to obtain an optimal solution.

$\Delta$ has two constraints: sum-to-one and nonnegativity. 
Let $\mu$ be a dual variable for $1-\sum_{a} \pi(a) = 0$ and $\lambda (a)$ be a dual variable for $\pi(a) \geq 0$.
Then, KKT conditions are as follows:
\begin{eqnarray}\label{eqn:kkt_mab}
\begin{aligned}
\forall \, i \;\; & 1-\sum_a \pi_q^{\star}(a) = 0,\; \pi_q^{\star}(a) \geq0 \\
\forall \, i \;\; & \lambda^{\star}(a) \leq 0 \\
\forall \, i \;\; & \lambda^{\star}(a)p_i^{\star} = 0 \\
\forall \, i \;\; & \mathbf{r}(a) - \mu^{\star} - \ln_{q} (\pi_q^{\star}(a)) -(\pi_q^{\star}(a))^{q-1} + \lambda^{\star}(a)= 0 
\end{aligned}
\end{eqnarray}
where $^{\star}$ indicates an optimal solution.
We focus on the last condition.
The last condition is converted into
\begin{eqnarray}\label{eqn:kkt_mab_stn}
\begin{aligned}
0&=\mathbf{r}(a) - \mu^{\star} - \ln_{q} (\pi_q^{\star}(a)) -(\pi_q^{\star}(a))^{q-1} + \lambda^{\star}(a)\\
0&=\mathbf{r}(a) - \mu^{\star} - \ln_{q} (\pi_q^{\star}(a)) -(q-1)\frac{\pi_q^{\star}(a)^{q-1}-1}{q-1}  - 1 + \lambda^{\star}(a)\\
0&=\mathbf{r}(a) - \mu^{\star} - q\ln_{q} (\pi_q^{\star}(a))  - 1 + \lambda^{\star}(a)
\end{aligned}
\end{eqnarray}

First, let's consider positive measure $\pi_{q}^{\star}(a) > 0$ ($\lambda^{\star}(a) = 0$).
Then, from equation (\ref{eqn:kkt_mab_stn}),
\begin{equation}
\exp_{q}\left(\frac{\mathbf{r}(a)}{q} - \frac{\mu^{\star}  + 1}{q}\right) = \pi_{q}^{\star}(a) 
\end{equation}
and $\mu^{\star}$ can be found by solving the following equation:
\begin{equation}\label{eqn:sum2one}
\sum_{a}\exp_{q}\left(\frac{\mathbf{r}(a)}{q} - \frac{\mu^{\star}  + 1}{q}\right) = 1.
\end{equation}
Since the equation (\ref{eqn:sum2one}) is exactly same as a normalization equation (\ref{def:qpoten}),
$\mu^{\star}$ can be found using a $q$-potential function $\psi_{q}$:
\begin{equation}
\mu^{\star} = q \psi_{q}\left(\frac{\mathbf{r}}{q}\right) - 1
\end{equation}
Then,
\begin{equation}
\pi_{q}^{\star}(a) = \exp_{q}\left(\frac{\mathbf{r}(a)}{q}- \psi_{q}\left(\frac{\mathbf{r}}{q}\right)\right).
\end{equation}

The optimal value of primal problem is 
\begin{eqnarray}
\begin{aligned}
&\mathop{\mathbbm{E}}_{a\sim \pi^{\star}_{q}}\left[ R \right] + S_{q} (\pi^{\star}_{q})= \sum_{a} \mathbf{r}(a) \pi_{q}^{\star}(a)  - \sum_{a} \left[\frac{\mathbf{r}(a)}{q}- \psi_{q}\left(\frac{\mathbf{r}}{q}\right)\right]\pi_{q}^{\star}(a)\\
&=(q-1) \sum_{a} \frac{\mathbf{r}(a)}{q}\exp_{q}\left(\frac{\mathbf{r}(a)}{q}- \psi_{q}\left(\frac{\mathbf{r}}{q}\right)\right) +  \psi_{q}\left(\frac{\mathbf{r}}{q}\right).
\end{aligned}
\end{eqnarray}

Finally, let's check the supporting set.
For $\pi_{q}^{\star}(a) > 0$, the following condition should be satisfied:
\begin{equation}
1+(q-1)\left(\frac{\mathbf{r}(a)}{q}- \psi_{q}\left(\frac{\mathbf{r}}{q}\right)\right)  > 0,
\end{equation}
where this condition comes from the definition of $\exp_{q}(x)$.
\end{proof}

\begin{prop}
When $q=1$ and $q=2$, $\psi_{1}, \pi_{1}^{\star}$ and $\psi_{2}, \pi_{2}^{\star}$ are computed as follows:
\begin{eqnarray*}
\begin{aligned}
\pi_{1}^{\star} = \exp(\mathbf{r}(a) - \psi_{1}(\mathbf{r}))\\
\psi_{1}(\mathbf{r}) =  \log\sum_{a}\exp(\mathbf{r}(a))
\end{aligned}
\end{eqnarray*}
and
\begin{eqnarray*}
\begin{aligned}
\pi_{2}^{\star}(a)=\left[1+\left(\frac{\mathbf{r}(a)}{2}-\psi_{2}\left(\frac{\mathbf{r}}{2}\right)\right)\right]_{+}\\
\psi_{2}\left(\frac{\mathbf{r}}{2}\right) =1+\frac{\sum_{a\in S} \mathbf{r}(a)/2-1}{|S|},
\end{aligned}
\end{eqnarray*}
where $S$ is a supporting set, i.e., $S=\{ a | \pi_{2}^{\star}(a) > 0\}$.
\end{prop}
\begin{proof}
Let us start from the KKT condition in the proof of Remark \ref{rem:opt_sol_bndt}.
When $q=1$,
Equation (\ref{eqn:sum2one}) becomes
\begin{equation*}
\sum_{a}\exp\left(\mathbf{r}(a) - \mu^{\star} - 1\right) = 1.
\end{equation*}
Then, 
\begin{eqnarray*}
\begin{aligned}
\sum_{a}\exp\left(\mathbf{r}(a)\right)&= \exp\left(\mu^{\star}+1\right)\\
\log\left(\sum_{a}\exp\left(\mathbf{r}(a)\right)\right)&= \mu^{\star}+1\\
\mu^{\star}+1&=\psi_{1}\left(\mathbf{r}(a)\right)
\end{aligned}
\end{eqnarray*}
Finally,
\begin{eqnarray*}
\begin{aligned}
\pi_{1}^{\star} = \exp(\mathbf{r}(a) - \psi_{1}(\mathbf{r}))\\
\psi_{1}(\mathbf{r}) =  \log\sum_{a}\exp(\mathbf{r}(a)).
\end{aligned}
\end{eqnarray*}
When $q=2$,
let us consider the supporting set $S$ of $\pi_{2}^{\star}$.
Then, 
\begin{eqnarray*}
\begin{aligned}
\sum_{a\in S}\exp_{2}\left(\frac{\mathbf{r}(a)}{2} - \frac{\mu^{\star} + 1}{2}\right) &= \sum_{a\in S}\left(1+\frac{\mathbf{r}(a)}{2} - \frac{\mu^{\star} + 1}{2}\right) = 1\\
\sum_{a\in S}\left(1+\frac{\mathbf{r}(a)}{2}\right) &= \sum_{a\in S}\left(\frac{\mu^{\star} + 1}{2}\right) + 1\\
\frac{\sum_{a\in S}\left(1+\frac{\mathbf{r}(a)}{2}\right) - 1}{|S|} &= \frac{\mu^{\star} + 1}{2}\\
\frac{\mu^{\star} + 1}{2} &= \psi_{2}\left(\frac{\mathbf{r}}{2}\right).
 \end{aligned}
\end{eqnarray*}
Thus,
$$
\psi_{2}\left(\frac{\mathbf{r}}{2}\right) = 1 + \frac{\sum_{a\in S}\mathbf{r}(a)/2 - 1}{|S|} 
$$
\end{proof}

\section{\textit{q}-Maximum: Bounded Approximation of Maximum}

Now, we prove the property of \textit{q}-maximum which is defined by $\mathop{q\textnormal{-max}}_{x}(f(x)) \triangleq  \max_{ P \in \Delta } \left[ \mathop{\mathbbm{E}}_{X\sim P}\left[ f(X) \right] + S_{q} (P)\right]$
\begin{theorem} \label{thm:bound}
For any function $f(x)$ defined on finite input space $\mathcal{X}$,
the \textit{q}-maximum satisfies the following inequalities.
\begin{eqnarray}\label{def:qmax_bnd}
\begin{aligned}
\mathop{q\textnormal{-max}}_{x}\left(f(x)\right) + \ln_{q} \left(1/|\mathcal{X}|\right)  \leq \max_{x}(f(x)) \leq \mathop{q\textnormal{-max}}_{x}\left(f(x)\right)
\end{aligned}
\end{eqnarray}
where $|\mathcal{X}|$ is a cardinality.
\end{theorem}
\begin{proof}
First, consider the lower bound.
Let $\Delta$ be a probability simplex. Then,
\begin{eqnarray}
\begin{aligned}
\mathop{q\textnormal{-max}}_{x}(f(x)) &=  \max_{ P \in \Delta } \left[ \mathop{\mathbbm{E}}_{X\sim P}\left[ f(X) \right] + S_{q} (P)\right] \leq \max_{ P \in \Delta } \mathop{\mathbbm{E}}_{X\sim P}\left[ f(X) \right] +  \max_{ P \in \Delta } S_{q} (P) \\
&= \max_{x}(f(x)) -  \ln_{q} \left(\frac{1}{|\mathcal{X}|}\right)
\end{aligned}
\end{eqnarray}
where $ S_{q} (P)$ has the maximum at an uniform distribution.

The upper bound can be proven using the similar technique.
Let $P'$ be the distribution whose probability is concentrated at a maximum element,
which means if $x=\arg\max_{x'}f(x')$, then, $P'(x) = 1$ and, otherwise, $P(x') = 0$.
If there are multiple maximum at $f(x)$, then, one of them can be arbitrarily chosen.
Then, the Tsallis entropy of $P'$ becomes zero since all probability mass is concentrated at a single instance, i.e., $S_{q}(P')=0$.
Then, the upper bound can be computed as follows:
\begin{eqnarray}
\begin{aligned}
\mathop{q\textnormal{-max}}_{x}(f(x))  &=   \max_{ P \in \Delta } \left[ \mathop{\mathbbm{E}}_{X\sim P}\left[ f(X) \right] + S_{q} (P)\right]\\
&\geq  \mathop{\mathbbm{E}}_{X\sim P'}\left[ f(X) \right] + S_{q} (P')= f\left(\arg\max_{x'}f(x')\right)=\max_{x} f(x).
\end{aligned}
\end{eqnarray}
\end{proof}


\section{Tsallis Bellman Optimality Equation}
Markov Decision Processes with Tsallis entropy maximization is formulated as follows.
\begin{eqnarray}
\begin{aligned}\label{def:tsallis_mdp}
	& \underset{\pi\in\Pi}{\text{maximize}}
	& &  \mathop{\mathbb{E}}_{\tau \sim P, \pi}\left[\sum_{t}^{\infty}\gamma^{t}\mathbf{R}_t \right]+ \alpha S_{q}^{\infty} (\pi)
\end{aligned}
\end{eqnarray}
In this section, we analyze the optimality condition of a Tsallis MDP.

\begin{theorem}\label{thm:tbo}
An optimal policy $\pi^{\star}_{q}$ and optimal value $V^{\star}_{q}$ sufficiently and necessarily satisfy the following Tsallis-Bellman optimality (TBO) equations:
\begin{eqnarray}
\begin{aligned}\label{eqn:tbo}
Q^{\star}_{q}(s,a) &=\mathop{\mathbbm{E}}_{s'\sim P}[ \mathbf{r}(s,a,s') + \gamma V^{\star}_{q}(s')|s,a] \\
V^{\star}_{q}(s) &= \mathop{q\textnormal{-max}}_{a}(Q^{\star}_{q}(s,a))\\
\pi^{\star}_{q}(a|s)  &= \exp_{q}\left(\frac{Q^{\star}_{q}(s,a)}{q} - \psi_{q}\left(\frac{Q^{\star}_{q}(s,\cdot)}{q}\right)\right),
\end{aligned}
\end{eqnarray}
where $\psi_q$ is a q-potential function.
\end{theorem}

Before starting proof,
we first remind two propositions and prove one lemma.
They are mainly employed to convert the optimization variable from $\pi$ to the state action visitation $\rho$.

\begin{prop} \label{rmk:bellman_flow}
Let a state visitation be $\rho_{\pi}(s) = \mathop{\mathbb{E}}_{\tau \sim P, \pi}\left[ \sum_{t=0}^{\infty} \gamma^{t} \mathbbm{1}(s_{t}=s) \right] $ and state action visitation be $\rho_{\pi}(s,a) = \mathop{\mathbb{E}}_{\tau \sim P, \pi}\left[ \sum_{t=0}^{\infty} \gamma^{t} \mathbbm{1}(s_{t}=s,a_{t}=a) \right] $ .
Following relationships hold.
\begin{eqnarray}
\begin{aligned}
\rho_{\pi}(s) = \sum_{a} \rho_{\pi}(s,a), \;\; \rho_{\pi}(s,a) = \rho_{\pi}(s)\pi(a|s)
\end{aligned}
\end{eqnarray}
\begin{eqnarray} \label{eqn:bellman_flow}
\begin{aligned}
\sum_{a}\rho_{\pi}(s, a) = d(s) + \sum_{s',a'} P(s|s',a')\rho_{\pi}(s', a'), \;\; \rho_{\pi}(s,a)
\end{aligned}
\end{eqnarray}
where Equation (\ref{eqn:bellman_flow}) is called Bellman Flow constraints.
\end{prop}
\begin{proof}
Proof can be found in \cite{puterman2014markov,syed2008allinearprogram}
\end{proof}
Proposition \ref{rmk:bellman_flow} tells us, for fixed policy $\pi$, $\rho_{\pi}$ satisfies Bellman Flow constraints.
Then, the next remark show the opposite direction
where if some function $\rho$ satisfies Bellman Flow constraints, then there exist an unique policy which induces $\rho$.

\begin{prop}[Theorem 2 of \citet{syed2008allinearprogram}]\label{thm:otocrrsp}
Let $\mathbf{M}$ be a set of state-action visitation measures, i.e., $\mathbf{M} \triangleq\{\rho| \forall s,\;a,\; \rho(s, a)\ge 0,\; \sum_{a}\rho(s, a) = d(s) + \sum_{s',a'} P(s|s',a')\rho(s', a')\}$.
If $\rho \in \mathbf{M}$, then it is a state-action visitation measure
for $\pi_{\rho}(a|s) \triangleq \frac{\rho(s, a)}{\sum_{a'}\rho(s, a')}$, and 
$\pi_{\rho}$ is the unique policy whose state-action visitation
measure is $\rho$.
\end{prop}
\begin{proof}
Proof can be found in \cite{puterman2014markov,syed2008allinearprogram}.
\end{proof}
Now, proposition \ref{rmk:bellman_flow} and \ref{thm:otocrrsp} tell us that a policy and state action visitation have the one-to-one correspondence.
In the following lemmas, we convert the optimization variable from $\pi$ to $\rho$ based on one-to-one correspondence.

\begin{lemma}\label{thm:obj_chg}
Let $$\bar{S}_{q}^{\infty}(\rho) = - \sum_{s,a} \rho(s, a) \ln_{q}\left(\frac{\rho(s, a)}{\sum_{a'}\rho(s, a')}\right).$$ 
Then, for any stationary policy $\pi \in \Pi$ and any state-action visitation
measure $\rho \in \mathbf{M}$, 
$S_{q}^{\infty}(\pi)=\bar{S}_{q}^{\infty}(\rho_{\pi})$ and $\bar{S}_{q}^{\infty}(\rho) = S_{q}^{\infty}(\pi_{\rho})$ hold.
\end{lemma} 
\begin{proof}
First, show that $S_{q}^{\infty}(\pi)=\bar{S}_{q}^{\infty}(\rho_{\pi})$.
\begin{eqnarray}
\begin{aligned}
S_{q}^{\infty}(\pi) &= \mathop{\mathbb{E}}_{\tau \sim P, \pi}\left[ \sum_{t=0}^{\infty} \gamma^{t} S_{q} ( \pi( \cdot | s_t ) ) \right] =\sum_{s} S_{q} ( \pi( \cdot | s ) ) \cdot \mathbb{E}_{\tau \sim P, \pi}\left[ \sum_{t=0}^{\infty} \gamma^{t} \mathbbm{1}(s_{t}=s) \right]\\
&= \sum_{s} S_{q} ( \pi( \cdot | s ) ) \rho_{\pi}(s) = \sum_{s,a} -\ln_{q}(\pi(a|s)) \pi(a|s)\rho_{\pi}(s)\\
&= \sum_{s,a} -\ln_{q}\left(\frac{\rho_{\pi}(s, a)}{\sum_{a'}\rho_{\pi}(s, a')}\right) \rho_{\pi}(s,a)=\bar{S}_{q}^{\infty}(\rho_{\pi}) \\
\end{aligned}
\end{eqnarray}
Next, show that $\bar{S}_{q}^{\infty}(\rho) = S_{q}^{\infty}(\pi_{\rho})$.
\begin{eqnarray}
\begin{aligned}
\bar{S}_{q}^{\infty}(\rho) =\sum_{s,a} -\ln_{q}\left(\frac{\rho(s, a)}{\sum_{a'}\rho(s, a')}\right) \rho(s,a)=\sum_{s,a} -\ln_{q}(\pi_{\rho}(a|s)) \pi_{\rho}(a|s)\rho(s)=S_{q}^{\infty}(\pi_{\rho})
\end{aligned}
\end{eqnarray}
\end{proof}

\begin{lemcorollary}\label{lemcor:same_optima}
The problem (\ref{prob:vis_mdp}) is equivalent to a Tsallis MDP,
which means if $\rho^{\star}$ is an optimal solution of (\ref{prob:vis_mdp}), then, $\pi_{\rho^{\star}}$ is an optimal solution of a Tsallis MDP and vice versa.
\end{lemcorollary} 
\begin{proof}
Let $\rho^{\star}$ be an optimal solution of (\ref{prob:vis_mdp}).
Assume that there exist another policyt $\pi'$ such that $J(\pi')+S_{q}^{\infty}(\pi') > J(\pi_{\rho^{\star}})+S_{q}^{\infty}(\pi_{\rho^{\star}})$
where $J(\pi)=\mathbbm{E}_{\tau\sim\pi,P}\left[\sum_{t=0}^{\infty}\gamma^{t}\mathbf{R}_{t}\right]$.
Then, $\sum_{s,a}\rho_{\pi'}(s,a)\mathbf{r}(s,a)+\bar{S}_{q}^{\infty}(\rho_{\pi'}) > \sum_{s,a}\rho_{\star}(s,a)\mathbf{r}(s,a)+\bar{S}_{q}^{\infty}(\rho_{\star})$.
It contradicts to the fact that $\rho^{\star}$ is the optimal solution of (\ref{prob:vis_mdp}).
Thus, for all $\pi$, $J(\pi)+S_{q}^{\infty}(\pi) \leq J(\pi_{\rho^{\star}})+S_{q}^{\infty}(\pi_{\rho^{\star}})$
which means $\pi_{\rho^{\star}}$ is the optimal policy.
The opposite direction also can be proven in the same way.
\end{proof}
Lemma \ref{thm:obj_chg} shows that $\bar{S}_{q}^{\infty}(\rho)$ and $S_{q}^{\infty}(\pi)$ has the same function value.
Thus, we can freely change the optimization variable from $\pi$ to $\rho$ since the optimal point does not change due to the Corollary \ref{lemcor:same_optima}.

\subsection{Variable Change}
Based on Proposition \ref{thm:otocrrsp} and Lemma \ref{thm:obj_chg},
we convert a Tsallis MDP problem to 
\begin{eqnarray}
\begin{aligned}\label{prob:vis_mdp}
	& \underset{\rho}{\text{maximize}}
	& & \sum_{s,a}\rho(s, a) \sum_{s'}\mathbf{r}(s,a,s')P(s'|s,a) - \sum_{s,a} \rho(s, a) \ln_{q}\left(\frac{\rho(s, a)}{\sum_{a'}\rho(s, a')}\right) \\
	& \text{subject to}
	& &\forall s,\;a,\; \rho(s, a)\ge 0,\; \sum_{a}\rho(s, a) = d(s) + \sum_{s',a'} P(s|s',a')\rho(s', a').
\end{aligned}
\end{eqnarray}
Now, the optimization variables in the problem (\ref{prob:vis_mdp}) is a state action visitation.
In the following lemmas, we show that the problem (\ref{prob:vis_mdp}) is concave with respect to a state action visitation.

\begin{lemma}\label{lem:concavity_tsent}
$\bar{S}_{q}^{\infty}(\rho)$ is concave function with respect to $\rho \in \mathbf{M}$
\end{lemma} 
\begin{proof}
Let us consider the function $f(x) = -x \ln_{q}(x)$ defined over $( x > 0 )$.
Second derivative of $d^{2}f(x)/dx^{2}$ is computed as 
\begin{eqnarray*}
\begin{aligned}
\frac{d^{2}f(x)}{dx^{2}} = -qx^{q-2} < 0\;\; ( x > 0 ).
\end{aligned}
\end{eqnarray*}
Since its second derivative is always negative on its domain, $f(x)$ is a concave function.
From this fact, we can show that $\bar{S}_{q}^{\infty}(\rho)$ is concave.
Proving the concavity is equivalent to show that for any $0<\lambda_{1},\lambda_{2}<1$ such that $\lambda_{1}+\lambda_{2}=1$, and for $\rho_{1},\rho_{2}\in\mathbf{M}$ the following inequality holds
$$
\bar{S}_{q}^{\infty}(\lambda_{1}\rho_{1}+\lambda_{2}\rho_{2}) > \lambda_{1}\bar{S}_{q}^{\infty}(\rho_{1}) +\lambda_{2} \bar{S}_{q}^{\infty}(\rho_{2})
$$
For notional simplicity, let $\tilde{\rho}$ be $\lambda_{1}\rho_{1}+\lambda_{2}\rho_{2}$
and define $\mu_{1} = \frac{\lambda_{1}\sum_{a'}\rho_{1}(s,a')}{\sum_{a'}\tilde{\rho}(s,a')}$ and $\mu_{2} = \frac{\lambda_{2}\sum_{a'}\rho_{2}(s,a')}{\sum_{a'}\tilde{\rho}(s,a')}$.
Note that from the definition, $\mu_{1}+\mu_{2}=1$.
It can be shown as follow:
\begin{eqnarray}\label{eqn:proof_concav}
\small
\begin{aligned}
\bar{S}_{q}^{\infty}(\lambda_{1}\rho_{1}+\lambda_{2}\rho_{2}) &= -\sum_{s,a} \tilde{\rho}(s, a) \ln_{q}\left(\frac{\lambda_{1}\rho_{1}(s,a)+\lambda_{2}\rho_{2}(s,a)}{\sum_{a'}\tilde{\rho}(s, a')}\right)\\
&= -\sum_{s,a} \tilde{\rho}(s, a) \ln_{q}\left(\frac{\mu_{1}\rho_{1}(s,a)}{\sum_{a'}\rho_{1}(s,a')}+\frac{\mu_{2}\rho_{2}(s,a)}{\sum_{a'}\rho_{2}(s,a')}\right)\\
&= -\sum_{s,a}\left(\sum_{a'} \tilde{\rho}(s, a')\frac{\lambda_{1}\rho_{1}(s,a)+\lambda_{2}\rho_{2}(s,a)}{\sum_{a'} \tilde{\rho}(s, a')} \right) \ln_{q}\left(\frac{\mu_{1}\rho_{1}(s,a)}{\sum_{a'}\rho_{1}(s,a')}+\frac{\mu_{2}\rho_{2}(s,a)}{\sum_{a'}\rho_{2}(s,a')}\right)\\
&= -\sum_{s,a}\sum_{a'} \tilde{\rho}(s, a')\left(\frac{\mu_{1}\rho_{1}(s,a)}{\sum_{a'} \rho_{1}(s,a')}+\frac{\mu_{2}\rho_{2}(s,a)}{\sum_{a'} \tilde{\rho}(s, a')}\right)  \ln_{q}\left(\frac{\mu_{1}\rho_{1}(s,a)}{\sum_{a'}\rho_{1}(s,a')}+\frac{\mu_{2}\rho_{2}(s,a)}{\sum_{a'}\rho_{2}(s,a')}\right)
\end{aligned}
\end{eqnarray}
Then, for all $s,a$,
\begin{eqnarray*}
\small
\begin{aligned}
&-\left(\frac{\mu_{1}\rho_{1}(s,a)}{\sum_{a'} \rho_{1}(s,a')}+\frac{\mu_{2}\rho_{2}(s,a)}{\sum_{a'} \tilde{\rho}(s, a')}\right)  \ln_{q}\left(\frac{\mu_{1}\rho_{1}(s,a)}{\sum_{a'}\rho_{1}(s,a')}+\frac{\mu_{2}\rho_{2}(s,a)}{\sum_{a'}\rho_{2}(s,a')}\right) \\
&> -\mu_{1}\frac{\rho_{1}(s,a)}{\sum_{a'} \rho_{1}(s,a')}\ln_{q}\left(\frac{\rho_{1}(s,a)}{\sum_{a'} \rho_{1}(s,a')}\right)-\mu_{2}\frac{\rho_{2}(s,a)}{\sum_{a'} \rho_{2}(s,a')}\ln_{q}\left(\frac{\rho_{2}(s,a)}{\sum_{a'} \rho_{2}(s,a')}\right)
\end{aligned}
\end{eqnarray*}
Equation (\ref{eqn:proof_concav}) becomes
\begin{eqnarray*}
\small
\begin{aligned}
\bar{S}_{q}^{\infty}(\lambda_{1}\rho_{1}+\lambda_{2}\rho_{2}) =& -\sum_{s,a}\sum_{a'} \tilde{\rho}(s, a')\left(\frac{\mu_{1}\rho_{1}(s,a)}{\sum_{a'} \rho_{1}(s,a')}+\frac{\mu_{2}\rho_{2}(s,a)}{\sum_{a'} \tilde{\rho}(s, a')}\right)  \ln_{q}\left(\frac{\mu_{1}\rho_{1}(s,a)}{\sum_{a'}\rho_{1}(s,a')}+\frac{\mu_{2}\rho_{2}(s,a)}{\sum_{a'}\rho_{2}(s,a')}\right)\\
>& -\sum_{s,a}\sum_{a'} \tilde{\rho}(s, a')\mu_{1}\frac{\rho_{1}(s,a)}{\sum_{a'} \rho_{1}(s,a')}\ln_{q}\left(\frac{\rho_{1}(s,a)}{\sum_{a'} \rho_{1}(s,a')}\right)\\
&-\sum_{s,a}\sum_{a'} \tilde{\rho}(s, a')\mu_{2}\frac{\rho_{2}(s,a)}{\sum_{a'} \rho_{2}(s,a')}\ln_{q}\left(\frac{\rho_{2}(s,a)}{\sum_{a'} \rho_{2}(s,a')}\right)
\end{aligned}
\end{eqnarray*}
Since $\sum_{a'} \tilde{\rho}(s, a')\mu_{1}\frac{\rho_{1}(s,a)}{\sum_{a'} \rho_{1}(s,a')} = \sum_{a'} \tilde{\rho}(s, a')\frac{\lambda_{1}\sum_{a'}\rho_{1}(s,a')}{\sum_{a'}\tilde{\rho}(s,a')}\frac{\rho_{1}(s,a)}{\sum_{a'} \rho_{1}(s,a')} = \lambda_{1}\rho_{1}(s,a)$,
finally, we get
\begin{eqnarray*}
\small
\begin{aligned}
\bar{S}_{q}^{\infty}(\lambda_{1}\rho_{1}+\lambda_{2}\rho_{2})>& -\sum_{s,a}\sum_{a'} \tilde{\rho}(s, a')\mu_{1}\frac{\rho_{1}(s,a)}{\sum_{a'} \rho_{1}(s,a')}\ln_{q}\left(\frac{\rho_{1}(s,a)}{\sum_{a'} \rho_{1}(s,a')}\right)\\
&-\sum_{s,a}\sum_{a'} \tilde{\rho}(s, a')\mu_{2}\frac{\rho_{2}(s,a)}{\sum_{a'} \rho_{2}(s,a')}\ln_{q}\left(\frac{\rho_{2}(s,a)}{\sum_{a'} \rho_{2}(s,a')}\right)\\
=&-\sum_{s,a} \lambda_{1}\rho_{1}(s,a) \ln_{q}\left(\frac{\rho_{1}(s,a)}{\sum_{a'} \rho_{1}(s,a')}\right)-\sum_{s,a} \lambda_{2}\rho_{2}(s,a) \ln_{q}\left(\frac{\rho_{2}(s,a)}{\sum_{a'} \rho_{2}(s,a')}\right)\\
=&\lambda_{1}\bar{S}_{q}^{\infty}(\rho_{1})+\lambda_{2}\bar{S}_{q}^{\infty}(\rho_{2})
\end{aligned}
\end{eqnarray*}
Note that this proof holds for every $q$ value greater than zero.
\end{proof}
\begin{lemcorollary}
The problem (\ref{prob:vis_mdp}) is concave with respect to $\rho\in\mathbf{M}$
\end{lemcorollary} 
\begin{proof}
The objective function of (\ref{prob:vis_mdp}) is concave function w.r.t $\rho$ since the first term is linear and the second term is concave be Lemma \ref{lem:concavity_tsent}.
All constraints are linear w.r.t $\rho$.
Thus, the problem is a concave problem.
\end{proof}

\subsection{Proof of Tsallis Bellman Optimality Equation}
\begin{proof}[Proof of Theorem \ref{thm:tbo}]
Since the problem (\ref{prob:vis_mdp}) is concave with respect to $\rho$,
the primal and dual solutions necessarily and sufficiently satisfy a KKT condition.
First, the Lagrangian objecitve $\mathcal{L}\triangleq\sum_{s,a}\rho(s, a) \mathbf{r}(s,a) - \sum_{s,a} \rho(s, a) \ln_{q}\left(\frac{\rho(s, a)}{\sum_{a'}\rho(s, a')}\right)+\sum_{s,a}\lambda(s,a)\rho(s, a) + \sum_{s} \mu(s)\left(d(s) + \sum_{s',a'} P(s|s',a')\rho(s', a') - \sum_{a}\rho(s, a)\right)$
where $\lambda(s,a)$ and $\mu(s)$ are dual variables for nonnegativity and Bellman flow constraints.
The KKT conditions of the problem (\ref{prob:vis_mdp}) are as follows:
\begin{eqnarray}\label{eqn:kkt}
\small
\begin{aligned}
\forall s,\;a,\; &\rho^{\star}(s,a) \ge 0,\; d(s) + \sum_{s',a'} P(s|s',a')\rho^{\star}(s',a') - \sum_{a}\rho^{\star}(s,a) = 0 \\
\forall s,\;a,\; & \lambda^{\star}(s,a) \leq 0 \\
\forall s,\;a,\; & \lambda^{\star}(s,a) \rho^{\star}(s,a) = 0 \\
\forall s,\;a,\; & 0 = \sum_{s'}\mathbf{r}(s,a,s')P(s'|s,a) + \gamma \sum_{s'}\mu^{\star}(s')P(s'|s,a) - \mu^{\star}(s) - q \ln_{q} \left(\frac{\rho^{\star}(s,a)}{\sum_{a'}\rho^{\star}(s,a')}\right) - 1\\
&+\sum_{a}\left(\frac{\rho^{\star}(s,a)}{\sum_{a'}\rho^{\star}(s,a')}\right)^{q} + \lambda^{\star}(s,a) 
\end{aligned}
\end{eqnarray}
We would like to note that the dervative of $\bar{S}_{q}^{\infty}(\rho)$ is computed as follows:
\begin{eqnarray}\label{eqn:tsallis_deriv}
\small
\begin{aligned}
\frac{\partial \bar{S}_{q}^{\infty}(\rho)}{\partial \rho(s'',a'')} &= - \sum_{s,a}\frac{\partial\rho(s,a)}{\partial \rho(s'',a'')} \ln_{q}\left(\frac{\rho(s,a)}{\sum_{a'}\rho(s,a')}\right)- \sum_{s,a}\rho(s,a)\frac{\partial \ln_{q}\left(\rho(s,a)/\sum_{a'}\rho(s,a')\right)}{\partial \rho(s'',a'')}\\
&= - \ln_{q}\left(\frac{\rho(s'',a'')}{\sum_{a'}\rho(s'',a')}\right) - \sum_{a}\rho(s'',a)\left(\frac{\rho(s'',a)}{\sum_{a'}\rho(s'',a')}\right)^{q-2}\left(\frac{\delta_{a''}(a)}{\sum_{a'}\rho(s'',a')} - \frac{\rho(s'',a)}{\left(\sum_{a'}\rho(s'',a')\right)^{2}}\right)\\
&= - \ln_{q}\left(\frac{\rho(s'',a'')}{\sum_{a'}\rho(s'',a')}\right) - \left(\frac{\rho(s'',a'')}{\sum_{a'}\rho(s'',a')}\right)^{q-1} + \sum_{a}\left(\frac{\rho(s'',a)}{\sum_{a'}\rho(s'',a')}\right)^{q}\\
&=  - q \ln_{q} \left(\frac{\rho(s'',a'')}{\sum_{a'}\rho(s'',a')}\right) - 1 +\sum_{a}\left(\frac{\rho(s'',a)}{\sum_{a'}\rho(s'',a')}\right)^{q}
\end{aligned}
\end{eqnarray}

Then, we show that $\mu^{\star}(s)$ is the same as optimal value $V^{\star}_{q}(s)$.
From the stationary condition, by multiplying $\pi_{\rho^{\star}}(a|s) = \rho^{\star}(s,a)/\sum_{a'}\rho^{\star}(s,a')$ and summing up with respect to $a$, the following equation is obtained:
\begin{eqnarray}\label{eqn:tsallis_deriv}
\small
\begin{aligned}
0&=\sum_{a} \sum_{s'}\mathbf{r}(s,a,s')P(s'|s,a)\pi_{\rho^{\star}}(a|s) + \gamma \sum_{s'}\mu^{\star}(s')\sum_{a}P(s'|s,a)\pi_{\rho^{\star}}(a|s) - \mu^{\star}(s) \\
&- q \sum_{a} \pi_{\rho^{\star}}(a|s)\ln_{q} \left(\frac{\rho^{\star}(s,a)}{\sum_{a'}\rho^{\star}(s,a')}\right) - 1 +\sum_{a}\pi_{\rho^{\star}}(a|s)\sum_{a''}\left(\frac{\rho^{\star}(s,a'')}{\sum_{a'}\rho^{\star}(s,a')}\right)^{q} + \sum_{a}\lambda^{\star}(s,a)\pi_{\rho^{\star}}(a|s)\\
&=\sum_{a}\sum_{s'}\mathbf{r}(s,a,s')P(s'|s,a)\pi_{\rho^{\star}}(a|s) + \gamma \sum_{s'}\mu^{\star}(s')\sum_{a}P(s'|s,a)\pi_{\rho^{\star}}(a|s) - \mu^{\star}(s) - q \sum_{a} \pi_{\rho^{\star}}(a|s)\ln_{q} \left(\pi_{\rho^{\star}}(a|s)\right)\\
& - 1 +\sum_{a''}\pi_{\rho^{\star}}(s,a)^{q} + \sum_{a}\lambda^{\star}(s,a)\pi_{\rho^{\star}}(a|s)\\
&=\sum_{a} \sum_{s'}\mathbf{r}(s,a,s')P(s'|s,a)\pi_{\rho^{\star}}(a|s) + \gamma \sum_{s'}\mu^{\star}(s')\sum_{a}P(s'|s,a)\pi_{\rho^{\star}}(a|s) - \mu^{\star}(s) - q \sum_{a} \pi_{\rho^{\star}}(a|s)\ln_{q} \left(\pi_{\rho^{\star}}(a|s)\right)\\
& - 1 +\sum_{a''}\pi_{\rho^{\star}}(s,a)^{q}\\
&=\sum_{a} \sum_{s'}\mathbf{r}(s,a,s')P(s'|s,a)\pi_{\rho^{\star}}(a|s) + \gamma \sum_{s'}\mu^{\star}(s')\sum_{a}P(s'|s,a)\pi_{\rho^{\star}}(a|s) - \mu^{\star}(s) - q \sum_{a} \pi_{\rho^{\star}}(a|s)\ln_{q} \left(\pi_{\rho^{\star}}(a|s)\right)\\
& + (q-1)\sum_{s,a}\pi_{\rho^{\star}}(s,a)\ln_{q}\left(\pi_{\rho^{\star}}(s,a)\right)\\
&=\sum_{a} \sum_{s'}\mathbf{r}(s,a,s')P(s'|s,a)\pi_{\rho^{\star}}(a|s) + \gamma \sum_{s'}\mu^{\star}(s')\sum_{a}P(s'|s,a)\pi_{\rho^{\star}}(a|s) - \mu^{\star}(s)-\sum_{s,a}\pi_{\rho^{\star}}(s,a)\ln_{q}\left(\pi_{\rho^{\star}}(s,a)\right).
\end{aligned}
\end{eqnarray}
Finally,
\begin{eqnarray}\label{eqn:mu_eqn}
\small
\begin{aligned}
\mu^{\star}(s)&=\sum_{a} \sum_{s'}\mathbf{r}(s,a,s')P(s'|s,a)\pi_{\rho^{\star}}(a|s) + \gamma \sum_{s'}\mu^{\star}(s')\sum_{a}P(s'|s,a)\pi_{\rho^{\star}}(a|s) - \sum_{a} \pi_{\rho^{\star}}(a|s)\ln_{q} \left(\pi_{\rho^{\star}}(a|s)\right)\\
&=\mathbbm{E}_{s'\sim P,a\sim\pi}\left[ \mathbf{r}(s,a,s') + \alpha S_{q}(\pi_{\rho^{\star}}(\cdot|s))+ \gamma \mu^{\star}(s')\middle|s\right]
\end{aligned}
\end{eqnarray}
This equation (\ref{eqn:mu_eqn}) exactly satisfies Tsallis Bellman expectation (TBE) equation of $\pi_{\rho^{\star}}$.
Thus, we want to claim that $\mu^{\star}(s)$ is the value $V^{\pi_{\rho^{\star}}}(s)$ of optimal policy $\pi_{\rho^{\star}}$, i.e., $\mu^{\star}(s)=V^{\star}_{q}(s)$.
However, to guarantee $\mu^{\star}(s)=V^{\star}_{q}(s)$, we should prove the following statement:
\textit{if an arbitrary function $f(s)$ satisfies a TBE equation for $\pi$, then, $f(s)=V^{\pi}(s)$},
which is not yet proven but it will be in Theorem \ref{thm:tpe}.
In this proof, let us just believe Theorem \ref{thm:tpe}.

Then, we first analyze a positive state-action visitation $\rho^{\star}(s,a) > 0$ $(\lambda^{\star}(s,a) = 0)$.
Using the fact that $\mu^{\star}=V_{q}^{\star}$, we can obtain $Q^{\star}_{q}(s,a)=\mathbbm{E}_{s'\sim P}[\mathbf{r}(s,a,s')+\gamma\mu^{\star}(s')]$.
By replacing $\rho^{\star}(s,a)/\sum_{a'}\rho^{\star}(s,a')$ with $\pi_{\rho^{\star}}(a|s)$ and using $Q^{\star}_{q}(s,a)=\mathbbm{E}_{s'\sim P}[\mathbf{r}(s,a,s')+\gamma\mu^{\star}(s')]$ and $\mu^{\star}(s)=V^{\star}(s)$,
\begin{eqnarray}
\begin{aligned}
&Q^{\star}_{q}(s,a) - V^{\star}_{q}(s) - q \ln_{q} \left(\pi_{\rho^{\star}}(a|s)\right) - 1 + \sum_{a}\pi_{\rho^{\star}}(a|s)^{q}= 0\\ 
&\frac{Q^{\star}_{q}(s,a)}{q} - \frac{V^{\star}_{q}(s) + 1 - \sum_{a}\left(\pi_{\rho^{\star}}(a|s)\right)^{q}}{q} = \ln_{q} \left(\pi_{\rho^{\star}}(a|s)\right)\\ 
&\exp_{q}\left(\frac{Q^{\star}_{q}(s,a)}{q} - \frac{V^{\star}_{q}(s) + 1 - \sum_{a}\left(\pi_{\rho^{\star}}(a|s)\right)^{q}}{q}\right) = \pi_{\rho^{\star}}(a|s).
\end{aligned}
\end{eqnarray}
Now, we can use $\sum_{a}\pi(a|s)=1$.
By summing up with respect to $a$, 
\begin{eqnarray}
\begin{aligned}
\sum_{a}\exp_{q}\left(\frac{Q^{\star}_{q}(s,a)}{q} - \frac{V^{\star}_{q}(s) + 1 - \sum_{a}\left(\pi_{\rho^{\star}}(a|s)\right)^{q}}{q}\right) = 1.
\end{aligned}
\end{eqnarray}
This equation is the normalization equation of \textit{q}-exponential distribution (\ref{def:qpoten}).
So, we can obtain the relationship between \textit{q}-potential and the optimal value function.
\begin{eqnarray}
\begin{aligned}
\psi_{q}\left(\frac{Q^{\star}_{q}(s,\cdot)}{q}\right) = \frac{V^{\star}_{q}(s) + 1 - \sum_{a}\left(\pi_{\rho^{\star}}(a|s)\right)^{q}}{q}
\end{aligned}
\end{eqnarray}
Finally, it is shown that the optimal policy has \textit{q}-exponential distribution of $Q^{\star}_{q}(s,\cdot)$.
\begin{eqnarray}
\begin{aligned}
\exp_{q}\left(\frac{Q^{\star}_{q}(s,a)}{q} - \psi_{q}\left(\frac{Q^{\star}_{q}(s,\cdot)}{q}\right)\right) = \pi_{\rho^{\star}}(a|s)
\end{aligned}
\end{eqnarray}

By plugging in this result into (\ref{eqn:mu_eqn}),
\begin{eqnarray}\label{eqn:value_eqn}
\begin{aligned}
V^{\star}_{q}(s)&=\sum_{a}\pi_{\rho^{\star}}(a|s)\sum_{s'}\left[ \mathbf{r}(s,a,s') + \gamma V^{\star}_{q}(s')P(s'|s,a)\right] - \sum_{a} \pi_{\rho^{\star}}(a|s)\ln_{q} \left(\pi^{\star}_{q}(a|s)\right)\\
&=\sum_{a}\pi_{\rho^{\star}}(a|s)Q^{\star}_{q}(s,a) -  \sum_{a} \pi_{\rho^{\star}}(a|s)\ln_{q} \left(\pi^{\star}_{q}(a|s)\right)\\
&=\sum_{a}\pi_{\rho^{\star}}(a|s)Q^{\star}_{q}(s,a) -  \sum_{a} \pi_{\rho^{\star}}(a|s)\left(\frac{Q^{\star}_{q}(s,a)}{q} - \psi_{q}\left(\frac{Q^{\star}_{q}(s,\cdot)}{q}\right)\right)\\
&=(q-1)\sum_{a}\pi_{\rho^{\star}}(a|s)\frac{Q^{\star}_{q}(s,a)}{q} + \psi_{q}\left(\frac{Q^{\star}_{q}(s,\cdot)}{q}\right)\\
&=\mathop{q\text{-max}}_{a'}\left(Q^{\star}_{q}(s,a')\right)
\end{aligned}
\end{eqnarray}
where the last equation is derived using the Equation (\ref{eqn:qmax_qpot}).

To summarize, we obtain the optimality condition for a Tsallis MDP as follows:
\begin{eqnarray}
\begin{aligned}\label{eqn:gtb}
Q^{\star}_{q}(s,a) &= \mathbbm{E}_{s'}\left[\mathbf{r}(s,a,s') + \gamma V^{\star}(s')\middle|s,a\right] \\
V^{\star}_{q}(s) &= \mathop{q\text{-max}}_{a'}(Q^{\star}_{q}(s,a'))\\
\pi^{\star}_{q}(a|s)  &= \exp_{q}\left(\frac{Q^{\star}_{q}(s,a)}{q} - \psi_{q}\left(\frac{Q^{\star}_{q}(s,\cdot)}{q}\right)\right)
\end{aligned}
\end{eqnarray}
We call these equations Tsallis Bellman optimality (TBO) equations.
\end{proof}

\section{Tsallis Policy Iteration}

\subsection{Tsallis Bellman Expectation (TBE) Equation}
In Tsallis policy evaluation,
for fixed $\pi$, the value functions of $\pi$ have the relationship as follows:
\begin{eqnarray}
\begin{aligned}\label{eqn:tbe}
Q^{\pi}_{q}(s,a) &= \mathop{\mathbbm{E}}_{s'\sim P}[\mathbf{r}(s,a,s') + \gamma V^{\pi}_{q}(s')|s,a] \\
V^{\pi}_{q}(s) &= \mathop{\mathbbm{E}}_{a\sim \pi}[Q^{\pi}_{q}(s,a) - \ln_{q}(\pi(a|s))],
\end{aligned}
\end{eqnarray}
These equations are derived from the definition of $V_{q}^{\pi}$ and $Q_{q}^{\pi}$.
Thus, if we have some value functions of Tsallis MDP, then, they satisfies TBE equation trivially.
However, main goal of Tsallis policy evaluation is to prove the opposite direction:
\textit{if an arbitrary function $f(s)$ satisfies a TBE equation for $\pi$, then, $f(s)=V^{\pi}(s)$}.

\subsection{Tsallis Bellman Expectation Operator and Tsallis Policy Evaluation}
\begin{eqnarray}
\begin{aligned}\label{def:tbe_op_q}
\left[\mathcal{T}_{q}^{\pi}F \right](s,a) &\triangleq \mathop{\mathbbm{E}}_{s' \sim P}[ \mathbf{r}(s,a,s') + \gamma V_{F}(s') |s,a]\\
V_{F}(s) &\triangleq \mathop{\mathbbm{E}}_{a\sim \pi}[F(s,a) - \ln_{q}(\pi(a|s))],
\end{aligned}
\end{eqnarray}
where $s'\sim P$ indicates $s'\sim P(\cdot|s,a)$  and $a'\sim\pi$ indicates $a'\sim\pi(\cdot|s)$.
Then, policy evaluation method in a Tsallis MDP can be simply defined as 
$$F_{k+1} = \mathcal{T}_{q}^{\pi}F_k.$$

\begin{theorem}[Tsallis Policy Evaluation]\label{thm:tpe}
For any fixed policy $\pi$ and entropic-index $q\geq1$, consider Tsallis Bellman expectation (TBE) operator $\mathcal{T}_{q}^{\pi}$,
and for an arbitrary initial function $F$ over $\mathcal{S}\times\mathcal{A}$, define Tsallis policy evaluation $F_{k+1}=\mathcal{T}_{q}^{\pi}F_{k}$. Then, $F_{k}$ converges into the $Q_{q}^{\pi}$ and satisfies TBE equation (\ref{eqn:tbe}).
In other words, the value function satisfying the TBE equation (\ref{eqn:tbe}) is unique.
\end{theorem}

Before proving Theorem \ref{thm:tpe},
we first drive the properties of $\mathcal{T}_{q}^{\pi}$.

\begin{lemma}\label{lem:pe_discounting}
For $F : \mathcal{S}\times\mathcal{A} \rightarrow R$ and $c \in R^{+}$, 
$ \mathcal{T}_{q}^{\pi}\left( F + c \mathbf{1}\right) = \mathcal{T}_{q}^{\pi}F + \gamma c \mathbf{1} $
where $\mathbf{1}:\mathcal{S}\times\mathcal{A} \rightarrow 1$
\end{lemma}
\begin{proof}
For all $s,a$,
\begin{eqnarray}
\small
\begin{aligned}
V_{F+c\mathbf{1}}(s) &= \mathop{\mathbbm{E}}_{a\sim \pi}[F(s,a) + c - \ln_{q}(\pi(a|s))]= \mathop{\mathbbm{E}}_{a\sim \pi}[F(s,a) - \ln_{q}(\pi(a|s))] + c=V_{F}(s) + c\\
\left[\mathcal{T}_{q}^{\pi}\left(F +  c \mathbf{1}\right) \right](s,a) &= \mathop{\mathbbm{E}}_{s' \sim P}[ \mathbf{r}(s,a,s') + \gamma V_{F +  c \mathbf{1}}(s')|s,a]= \mathop{\mathbbm{E}}_{s' \sim P}[ \mathbf{r}(s,a,s') + \gamma V_{F}(s') + \gamma c|s,a]\\
&= \mathop{\mathbbm{E}}_{s' \sim P}[ \mathbf{r}(s,a,s') + \gamma V_{F}(s')|s,a] + \gamma c= \mathcal{T}_{q}^{\pi}F(s) + \gamma c
\end{aligned}
\end{eqnarray}
\end{proof}

\begin{lemma}\label{lem:pe_monotone}
For $F, G : \mathcal{S}\times\mathcal{A} \rightarrow R$ and $F \succeq G$, 
$\mathcal{T}_{q}^{\pi}\left(F\right) \succeq \mathcal{T}_{q}^{\pi}\left(G\right)$
where $\succeq$ indicates $F(s,a)\geq G(s,a)$ for all $s,a$.
\end{lemma}
\begin{proof}
For all $s,a$,
\begin{eqnarray}
\small
\begin{aligned}
V_{F}(s) &= \mathop{\mathbbm{E}}_{a\sim \pi}[F(s,a) - \ln_{q}(\pi(a|s))]< \mathop{\mathbbm{E}}_{a\sim \pi}[G(s,a) - \ln_{q}(\pi(a|s))] = V_{G}(s)\\
\left[\mathcal{T}_{q}^{\pi}F \right](s,a) &= \mathop{\mathbbm{E}}_{s' \sim P}[ \mathbf{r}(s,a,s') + \gamma V_{F}(s')|s,a]< \mathop{\mathbbm{E}}_{s' \sim P}[ \mathbf{r}(s,a,s') + \gamma V_{G}(s')|s,a]=\left[\mathcal{T}_{q}^{\pi}G \right](s,a)
\end{aligned}
\end{eqnarray}
\end{proof}

\begin{lemma}\label{lem:pe_contraction}
$\mathcal{T}_{q}^{\pi}$ is $\gamma$-contraction mapping in $(C(\mathcal{S}\times\mathcal{A},R),|\cdot|_{\infty})$
where $C(\mathcal{S}\times\mathcal{A},R)\triangleq\{ F: \mathcal{S}\times\mathcal{A}\rightarrow R\}$ and $|F-G|_{\infty} = \sup_{s,a}|F(s,a)-G(s,a)|$
\end{lemma}
\begin{proof}
Let $d = |F-G|_{\infty}$. The, $G - d\mathbf{1} \succeq F\succeq G + d\mathbf{1}$.
From Lemma \ref{lem:pe_monotone},
$\mathcal{T}_{q}^{\pi}(G + d\mathbf{1}) \succeq \mathcal{T}_{q}^{\pi}F \succeq \mathcal{T}_{q}^{\pi}(G - d\mathbf{1})$.
From Lemma \ref{lem:pe_discounting},
$\mathcal{T}_{q}^{\pi}G + \gamma d \mathbf{1} \succeq \mathcal{T}_{q}^{\pi}F \succeq \mathcal{T}_{q}^{\pi}G - \gamma d \mathbf{1}$.
Then,$\gamma d \mathbf{1} \succeq \mathcal{T}_{q}^{\pi}F - \mathcal{T}_{q}^{\pi}G \succeq - \gamma d \mathbf{1}$.
Finally,
$$
|\mathcal{T}_{q}^{\pi}F-\mathcal{T}_{q}^{\pi}G|_{\infty} \leq \gamma d = \gamma |F-G|_{\infty}.
$$
Consequently,  $\mathcal{T}_{q}^{\pi}$ is $\gamma$-contraction.
\end{proof}

\subsubsection{Proof of Tsallis Policy Evaluation}
\begin{proof}[Proof of Theorem \ref{thm:tpe}]
From Lemma \ref{lem:pe_contraction}, $\mathcal{T}_{q}^{\pi}$ is $\gamma$-contraction and has an unique fixed point $F_{*}=\mathcal{T}_{q}^{\pi}F_{*}$ from the Banach fixed point theorem.
Then, for any initial function $F$, a sequence of $F_{k}$ converges to the fixed point, i.e., $F_{*}=\lim_{k\to\infty}(\mathcal{T}_{q}^{\pi})^{k}F_{0}$.
The fixed point $F_{*}$ satisfies a TBE equation as follows:
\begin{eqnarray}
\begin{aligned}
F_{*}(s,a) &=  \mathop{\mathbbm{E}}_{s'\sim P}[ \mathbf{r}(s,a,s') + \gamma V_{F_{*}}(s')|s,a] \\
V_{F_{*}}(s) &= \mathop{\mathbbm{E}}_{a\sim \pi}[F_{*}(s,a) - \ln_{q}(\pi(a|s))],
\end{aligned}
\end{eqnarray}
Since $F_{*}$ is unique, $F_{*}$ is the only function which satisfies a TBE equation. Thus, $F_{*}=Q_{q}^{\pi}$.
\end{proof}

\subsection{Tsallis Policy Improvement}
The value function evaluated from Tsallis policy evaluation
can be employed to update the policy distribution.
In policy improvement step,
the policy will be updated to maximize 
\begin{eqnarray}
\begin{aligned}\label{def:tp_imp}
\forall s, \, \pi_{k+1}(\cdot|s) \triangleq \arg\max_{\pi(\cdot|s)} &\mathop{\mathbbm{E}}_{a\sim \pi}[ Q^{\pi_{k}}_{q}(s,a) - \ln_{q}(\pi(a|s))|s]
\end{aligned}
\end{eqnarray}
\begin{theorem}[Tsallis Policy Improvement]\label{thm:tpi}
Let $\pi_{k+1}$ be the updated policy from (\ref{def:tp_imp}) using $Q_{q}^{\pi_{k}}$.
For all $(s,a)\in\mathcal{S}\times\mathcal{A}$, $Q^{\pi_{k+1}}_{q}(s,a)$ is greater than $Q^{\pi_{k}}_{q}(s,a)$.
\end{theorem}
\begin{proof}, unless $\pi_{k} = \pi_{k+1}$
Since $\pi_{k+1}$ is updated by maximizing Equation (\ref{def:tp_imp}) and the maximization in Equation (\ref{def:tp_imp}) is concave with respect to $\pi$, the following inequality holds 
\begin{equation}\label{eqn:tp_imp_res}
\mathop{\mathbbm{E}}_{a \sim \pi_{k+1}}\left[Q^{\pi_{k}}_{q}(s,a) - \ln_{q}(\pi_{k+1}(a|s))\middle|s\right] \geq \mathop{\mathbbm{E}}_{a \sim \pi_{k}}\left[Q^{\pi_{k}}_{q}(s,a) - \ln_{q}(\pi_{k}(a|s))\middle|s\right] = V^{\pi_{k}}_{q}(s),
\end{equation}
where the equality holds when $\pi_{k+1} = \pi_{k}$.
This inequality induces a performance improvement,
\begin{eqnarray}\label{eqn:quality}
\small
\begin{aligned}
Q^{\pi_{k}}_{q}(s,a) &= \mathop{\mathbbm{E}}_{  s_{1} \sim P}\left[ r(s_{0},a_{0},s_{1}) + \gamma V^{\pi_{k}}_{q}(s_{1})\middle| s_{0}=s, a_{0}=a \right]\\
&\leq \mathop{\mathbbm{E}}_{  s_{1} \sim P}\left[ r(s_{0},a_{0},s_{1})\middle| s_{0}=s, a_{0}=a \right]\\
&\;\;\;+ \gamma \mathop{\mathbbm{E}}_{ s_{1},a_{1} \sim P,\pi_{k+1}}\left[Q^{\pi_{k}}_{q}(s_{1},a_{1}) - \ln_{q}(\pi_{k+1}(a_{1}|s_{1}))\middle|s_{0}=s, a_{0}=a\right]\\
&= \mathop{\mathbbm{E}}_{ s_{1} \sim P}\left[ r(s_{0},a_{0},s_{1})\middle| s_{0}=s, a_{0}=a \right]\\
&\;\;\;+ \gamma \mathop{\mathbbm{E}}_{ s_{1:2},a_{1} \sim P,\pi_{k+1}}\left[ r(s_{1},a_{1},s_{2}) - \ln_{q}(\pi_{k+1}(a_{1}|s_{1})) + \gamma V^{\pi_{k}}_{q}(s_{2}) \middle|s_{0}=s, a_{0}=a\right] \\
&\leq \mathop{\mathbbm{E}}_{  s_{1} \sim P}\left[ r(s_{0},a_{0},s_{1})\middle| s_{0}=s, a_{0}=a \right]\\
&\;\;\;+ \gamma \mathop{\mathbbm{E}}_{  s_{1:t+1},a_{1:t} \sim P,\pi_{k+1}}\left[ \sum_{k=1}^{t} \gamma^{k-1} \left(r(s_{k},a_{k},s_{k+1}) -  \ln_{q}(\pi_{k+1}(a_{k}|s_{k}) \right)\middle|s_{0}=s, a_{0}=a\right]\\
&\;\;\; + \gamma^{t+1} \mathop{\mathbbm{E}}_{  s_{t+1} \sim P, \pi_{k+1}}\left[ V^{\pi_{k}}_{q}(s_{t+1}) \middle| s_{0}=s, a_{0}=a \right] \\
&\vdots \\
&\leq \mathop{\mathbbm{E}}_{  s_{1} \sim P}\left[ r(s_{0},a_{0},s_{1}) + \gamma V_{q}^{\pi_{k+1}}(s_{1})\middle| s_{0}=s, a_{0}=a \right] = Q^{\pi_{k+1}}_{q}(s,a),
\end{aligned}
\end{eqnarray}
where $\gamma^{t+1} \mathop{\mathbbm{E}}_{  s_{t+1} \sim P, \pi_{k+1}}\left[ V^{\pi_{k}}_{q}(s_{t+1}) \middle| s_{0}=s, a_{0}=a \right]\rightarrow0$ as $t\rightarrow \infty$.
\end{proof}

\begin{theorem}[Optimality of TPI]\label{thm:tpi}
TPI converges into an optimal policy and value of a Tsallis MDP.
\end{theorem}
\begin{proof}
From the fact that reward function $\mathbf{r}$ has upper bound $\mathbf{r}_{\max}$ and $\mathcal{S}\times\mathcal{A}$ is bounded,
$Q^{\pi_{k}}_{q}$ is also bouned.
Then, since a sequence of $Q^{\pi_{k}}_{q}$ is monotonically non-decreasing and bounded, it converges to some point $\pi_{*}$.
Now, proof will be done by showing $\pi_{*}=\pi^{\star}_{q}$.
First, from the policy improvement,
We have $\pi_{*}(\cdot|s) = \arg\max_{\pi(\cdot|s)} \mathop{\mathbbm{E}}_{a\sim \pi}[ Q^{\pi_{*}}_{q}(s,a) - \ln_{q}(\pi(a|s))|s]$ and at $\pi_{*}$, the equality in Equation (\ref{eqn:tp_imp_res}) holds, i.e., $V^{\pi_{*}}_{q}(s)=\mathop{\mathbbm{E}}_{a \sim \pi_{*}}\left[Q^{\pi_{*}}_{q}(s,a) - \alpha \ln_{q}(\pi_{*}(a|s))\middle|s\right]$.
Then, the following equality holds,
$$
V^{\pi_{*}}_{q}(s)=\max_{\pi(\cdot|s)}\mathop{\mathbbm{E}}_{a \sim \pi}\left[Q^{\pi_{*}}_{q}(s,a) - \alpha \ln_{q}(\pi(a|s))\middle|s\right],
$$
which is equivalent to $V^{\pi_{*}}_{q}(s)=\mathop{q\text{-max}}_{a'}Q^{\pi_{*}}(s,a')$.
It can be also known that $\pi_{*}$ is the solution of $q$-maximum.
From the TBE equation, $Q^{\pi_{*}}_{q}(s,a) = \mathop{\mathbbm{E}}_{s'\sim P}[ \mathbf{r}(s,a,s') + \gamma V^{\pi_{*}}_{q}(s')|s,a]$.
Thus, $\pi_{*}$ satisfies a TBO equation and by Theorem \ref{thm:tbo}, $\pi_{*}=\pi^{\star}_{q}$.
\end{proof}

\subsection{Tsallis Value Iteration}
Tsallis value iteration is derived from the optimality equation.
From TBO equation, Tsallis Bellman optimality operator is defined by
\begin{eqnarray}
\begin{aligned}\label{def:tbo_op}
\left[\mathcal{T}_{q}F \right](s,a) &\triangleq \mathop{\mathbbm{E}}_{s'\sim P}\left[\mathbf{r}(s,a,s') + \gamma V_{F}(s)\middle|s,a\right]\\
V_{F}(s) &\triangleq \mathop{q\text{-max}}_{a'}\left(F(s,a')\right).
\end{aligned}
\end{eqnarray}
Then, a Tsallis value iteration is defined by repeatedly applying TBO operator: 
$$F_{k+1} = \mathcal{T}_{q}F_{k}.$$

\begin{theorem}\label{thm:optimality}
For any fixed entropic-index $q\geq1$, consider Tsallis Bellman optimality (TBO) operator $\mathcal{T}_{q}$,
and for an arbitrary initial function $F_{0}$ over $\mathcal{S}\times\mathcal{A}$, define Tsallis value iteration $F_{k+1}=\mathcal{T}_{q}F_{k}$. Then, $F_{k}$ converges into the $Q_{q}^{\star}$.
\end{theorem}

Before proving Theorem \ref{thm:optimality},
we first drive the properties of $q$-maximum and $\mathcal{T}_{q}$.

\begin{lemma}\label{lem:qmax_prop}
For any function $f(x)$ defined on finite input space $\mathcal{X}$ and $c\in R$,
The following equality hold:
\begin{enumerate}
\item $\mathop{q\textnormal{-max}}_{x}( f(x)  + c\mathbf{1}) = \mathop{q\textnormal{-max}}_{x}( f(x) )+c$ \label{prop1}
\item $\mathop{q\textnormal{-max}}_{x}( f(x))$ is monotone. If $f \preceq g$, then $\mathop{q\textnormal{-max}}_{x}(x) \leq \mathop{q\textnormal{-max}}_{x}( y)$ \label{prop2}
\end{enumerate}
where $\mathbf{1}$ is a constant function whose value is one. 
\end{lemma}
\begin{proof}
For property \ref{prop1},
\begin{eqnarray}
\begin{aligned}
\mathop{q\textnormal{-max}}_{x}(f(x) + c\mathbf{1}) &=  \max_{ P \in \Delta } \left[ \mathop{\mathbbm{E}}_{X\sim P}\left[ f(X)  + c\mathbf{1}(X) \right] + S_{q} (P)\right]\\
&=  \max_{ P \in \Delta } \left[ \mathop{\mathbbm{E}}_{X\sim P}\left[ f(X) \right] + c + S_{q} (P)\right]\\
&=  \max_{ P \in \Delta } \left[ \mathop{\mathbbm{E}}_{X\sim P}\left[ f(X) \right] + S_{q} (P)\right] + c=\mathop{q\textnormal{-max}}_{x}(f(x))+ c
\end{aligned}
\end{eqnarray}
For property \ref{prop2},
\begin{eqnarray}
\begin{aligned}
\mathop{q\textnormal{-max}}_{x}(f(x)) &=\max_{ P \in \Delta } \left[ \mathop{\mathbbm{E}}_{X\sim P}\left[ f(X) \right] + S_{q} (P)\right]\\
&=\mathop{\mathbbm{E}}_{X\sim P^{\star}(f)}\left[ f(X) \right] + S_{q} (P^{\star}(f))\leq \mathop{\mathbbm{E}}_{X\sim P^{\star}(f)}\left[ g(X) \right] + S_{q} (P^{\star}(f)) \;\;\; (\because f \preceq g) \\
&\leq\max_{ P' \in \Delta } \left[ \mathop{\mathbbm{E}}_{X\sim P'}\left[ g(X) \right] + S_{q} (P')\right]= \mathop{q\textnormal{-max}}_{x}(f(x)),
\end{aligned}
\end{eqnarray}
where $P^{\star}(f)$ indicates the optimal distribution of $q\text{-max}_{x}(f(x))$.
\end{proof}

\begin{lemma}\label{lemma:discounting}
For $F : \mathcal{S}\times\mathcal{A} \rightarrow R$ and $c \in R$, 
$ \mathcal{T}_{q}\left( F + c \mathbf{1}\right) = \mathcal{T}_{q}F + \gamma c \mathbf{1} $
where $\mathbf{1}:\mathcal{S}\times\mathcal{A} \rightarrow 1$
\end{lemma}
\begin{proof}
For all $s,a$,
\begin{eqnarray}
\begin{aligned}
V_{F+c\mathbf{1}}(s) &= \mathop{q\text{-max}}_{a'}\left(F(s,a')+c\right) = \mathop{q\text{-max}}_{a'}\left(F(s,a')\right)+c = V_{F}(s)+c\\
\left[\mathcal{T}_{q}F + c \mathbf{1}\right](s,a) &= \mathop{\mathbbm{E}}_{s'\sim P}\left[r(s,a,s') + \gamma V_{F+c\mathbf{1}}(s')\middle|s,a\right]\\
 &= \mathop{\mathbbm{E}}_{s'\sim P}\left[r(s,a,s') + \gamma  V_{F}(s') + \gamma c\middle|s,a\right]\\
 &= \mathop{\mathbbm{E}}_{s'\sim P}\left[r(s,a,s') + \gamma V_{F}(s')\middle|s,a\right] + \gamma c = \left[\mathcal{T}_{q}F\right](s,a)  + \gamma c\\
\end{aligned}
\end{eqnarray}
\end{proof}

\begin{lemma}\label{lemma:monotone}
For $F, G : \mathcal{S}\times\mathcal{A} \rightarrow R$ and $F \succeq G$, 
$\mathcal{T}_{q}\left(F\right) \succeq \mathcal{T}_{q}\left(G\right)$
where $\succeq$ indicates $F(s,a)\geq G(s,a)$ for all $s,a$.
\end{lemma}
\begin{proof}
For all $s,a$,
\begin{eqnarray}
\begin{aligned}
V_{F}(s) &= \mathop{q\text{-max}}_{a'}\left(F(s,a')\right) \leq \mathop{q\text{-max}}_{a'}\left(G(s,a')\right) = V_{G}(s)\\
\left[\mathcal{T}_{q}F\right](s,a) &= \mathop{\mathbbm{E}}_{s'\sim P}\left[r(s,a,s') + \gamma V_{F}(s')\middle|s,a\right] \\
&\leq \mathop{\mathbbm{E}}_{s'\sim P}\left[r(s,a,s') + \gamma V_{G}(s')\middle|s,a\right]=\left[\mathcal{T}_{q}G\right](s,a) 
\end{aligned}
\end{eqnarray}
\end{proof}

\begin{lemma}\label{lemma:contraction}
$\mathcal{T}_{q}$ is $\gamma$-contraction mapping in $(C(\mathcal{S}\times\mathcal{A},R),|\cdot|_{\infty})$
where $C(\mathcal{S}\times\mathcal{A},R)\triangleq\{ F: \mathcal{S}\times\mathcal{A}\rightarrow R\}$ and $|F-G|_{\infty} = \sup_{s,a}|F(s,a)-G(s,a)|$
\end{lemma}
\begin{proof}
Let $d = |F-G|_{\infty}$. The, $G - d\mathbf{1} \succeq F\succeq G + d\mathbf{1}$.
From Lemma \ref{lem:pe_monotone},
$\mathcal{T}_{q}(G + d\mathbf{1}) \succeq \mathcal{T}_{q}F \succeq \mathcal{T}_{q}(G - d\mathbf{1})$.
From Lemma \ref{lem:pe_discounting},
$\mathcal{T}_{q}G + \gamma d \mathbf{1} \succeq \mathcal{T}_{q}F \succeq \mathcal{T}_{q}G - \gamma d \mathbf{1}$.
Then,$\gamma d \mathbf{1} \succeq \mathcal{T}_{q}F - \mathcal{T}_{q}G \succeq - \gamma d \mathbf{1}$.
Finally,
$$
|\mathcal{T}_{q}F-\mathcal{T}_{q}G|_{\infty} \leq \gamma d = \gamma |F-G|_{\infty}.
$$
Consequently,  $\mathcal{T}_{q}$ is $\gamma$-contraction.
\end{proof}

\subsubsection{Proof of Tsallis Value Iteration}
\begin{proof}[Proof of Theorem \ref{thm:optimality}]
From Lemma \ref{lemma:contraction}, $\mathcal{T}_{q}$ is $\gamma$-contraction and has an unique fixed point $F_{*}=\mathcal{T}_{q}F_{*}$ from the Banach fixed point theorem.
Then, for any initial function $F$, a sequence of $F_{k}$ converges to the fixed point, i.e., $F_{*}=\lim_{k\to\infty}(\mathcal{T}_{q})^{k}F_{0}$.
The fixed point $F_{*}$ satisfies a TBO equation as follows:
\begin{eqnarray}
\begin{aligned}
F_{*}(s,a) &= \mathop{\mathbbm{E}}_{s'\sim P}[ \mathbf{r}(s,a,s') + \gamma V_{F_{*}}(s')|s,a] \\
V_{F_{*}}(s) &= \mathop{q\text{-max}}_{a}[F_{*}(s,a)],
\end{aligned}
\end{eqnarray}
Since TBO equation is the necessary and sufficient conditions,$F_{*}=Q_{q}^{\star}$.
\end{proof}

\subsection{Performance Error Bounds}
\begin{theorem}\label{thm:error_bounds}
Let $J(\pi)$ be the expected sum of rewards of a given policy $\pi$, $\pi^{\star}$ be the optimal policy of an original MDP, and $\pi^{\star}_{q}$ be the optimal policy of a Tsallis MDP with an entropic index \textit{q}.
Then, the following inequality holds,
\begin{eqnarray}
\small
\begin{aligned}
J(\pi^{\star}) + (1-\gamma)^{-1}\ln_{q} \left(1/|\mathcal{A}|\right) \leq J(\pi_{q}^{\star}) \leq J(\pi^{\star}).
\end{aligned}
\end{eqnarray}
where $|\mathcal{A}|$ is the cardinality.
\end{theorem}
\begin{lemma}\label{lemma:tbo_lowerbnd}
Let $$
[\mathcal{T}F](s,a) \triangleq \mathop{\mathbbm{E}}_{s'\sim P}[ \mathbf{r}(s,a,s') + \gamma \max_{a'}F(s',a')|s,a]
$$ for a function $F$.
$\mathcal{T}$ is the original Bellman optimality operator which is used for an original value iteration.
Then, for all positive integer $k$ and any function $F$ over $\mathcal{S}\times\mathcal{A}$,
$$
\mathcal{T}_{q}^{k}F \succeq \mathcal{T}^{k}F
$$
where $\mathcal{T}^k$ indicates $k$ tiems application of $\mathcal{T}$.
Furthermore, $V^{\star}_{q} \succeq V^{\star}$ holds
which means that the optimal value of Tsallis MDP is greater than the optimal value of the original MDP.
\end{lemma}
\begin{proof}
When $k=1$, from Lemma \ref{lem:qmax_prop}, for all $s,a$,
\begin{eqnarray}
\begin{aligned}
\left[ \mathcal{T} F \right](s,a)&=\mathop{\mathbbm{E}}_{s'\sim P}[ \mathbf{r}(s,a,s') + \gamma \max_{a'}F(s',a')|s,a]\\
&\leq\mathop{\mathbbm{E}}_{s'\sim P}\left[ \mathbf{r}(s,a,s') + \gamma \mathop{q\text{-max}}_{a'}F(s',a')|s,a\right]=[ \mathcal{T}_{q}F ](s,a)
\end{aligned}
\end{eqnarray}
Now, assume that the statement holds when $k=n$,
then,
\begin{eqnarray}
\begin{aligned}
\mathcal{T}^{n+1} F = \mathcal{T} \mathcal{T}^{n} F \preceq \mathcal{T}_{q} \mathcal{T}^{n} F \preceq \mathcal{T}_{q} \mathcal{T}_{q}^{n} F \preceq \mathcal{T}_{q}^{n+1} F
\end{aligned}
\end{eqnarray}
From mathematical induction, the statement holds for all positive integers.
Furthermore, 
$$
V^{\star}=\lim_{k\to\infty}\mathcal{T}^{k} F \preceq \lim_{k\to\infty} \mathcal{T}_{q}^{k} F = V^{\star}_{q}
$$
\end{proof}
We would like to note that the gap between $V^{\star}$ and $V_{q}^{\star}$ is induced from the Tsallis entropy.

\subsubsection{Proof of Performance Error Bounds}
\begin{proof}[Proof of Theorem \ref{thm:error_bounds}]
The upper bound is trivial.
Since the original MDP maximizes $J(\pi)$ without the entropy maximization,
it is clear that $J(\pi^{\star}_{q}) \leq J(\pi^{\star})$
where $J(\pi)\triangleq\mathbbm{E}_{\tau\sim\pi,P}[\sum_{t=0}^{\infty}\gamma^{t}\mathbf{R}_{t}]$.
For the lower bound, using Lemma \ref{lemma:tbo_lowerbnd},
\begin{eqnarray}
\begin{aligned}
J(\pi^{\star}) = \mathbbm{E}_{s_{0}\sim d}\left[V^{\star}(s_{0})\right] &\leq \mathbbm{E}_{s_{0}\sim d}\left[V^{\star}_{q}(s_{0})\right] = J(\pi^{\star}_{q}) + S_{q}^{\infty}\left(\pi_{q}^{\star}\right)\\
&\leq J(\pi^{\star}_{q}) + \mathbbm{E}_{\tau\sim\pi,P}\left[\sum_{t=0}^{\infty}\gamma^{t}S_{q}\left(\pi_{q}^{\star}(\cdot|s_{t})\right)\right]\\
&\leq J(\pi^{\star}_{q}) + \mathbbm{E}_{\tau\sim\pi,P}\left[\sum_{t=0}^{\infty}\gamma^{t}\max_{\pi(\cdot|s_{t})}S_{q}\left(\pi(\cdot|s_{t})\right)\right]\\
&\leq J(\pi^{\star}_{q}) - \mathbbm{E}_{\tau\sim\pi,P}\left[\sum_{t=0}^{\infty}\gamma^{t}\ln_{q}\left(1/|\mathcal{A}|\right)\right]\\
&\leq J(\pi^{\star}_{q}) -(1-\gamma)^{-1}\ln_{q}\left(1/|\mathcal{A}|\right)
\end{aligned}
\end{eqnarray}
\end{proof}

\section{Continuous State and Action Spaces}
Main difference between continuous space and finite discrete space is the number of optimization variables.
For discrete (and finite) state and action space, $\rho$ (or $\pi$) can be represented as $|\mathcal{S}|\times|\mathcal{A}|$ dimensional vector whose element is $\rho(s,a)$ (or $\pi$, respectively).
However, when state and action spaces are continuous,
$\rho$ (or $\pi$) become an infinite dimensional vector.
In other words, $\rho$ becomes a function over $\mathcal{S}\times\mathcal{A}$ which satisfies $$\mathbf{M} \triangleq\left\{\rho\middle| \forall s,\;a,\; \rho(s, a)\ge 0,\; \int_{a}\rho(s, a)da = d(s) + \int_{s',a'} P(s|s',a')\rho(s', a')ds'da'\right\}.$$
Note that all summations in the aforementioned theorems are changed to the integral as the state and action are continuous now.
Then, our optimization problem is converted to the integral form.
\begin{eqnarray}
\small
\begin{aligned}
	& \underset{\rho}{\text{maximize}}
	& & \int_{s,a}\rho(s, a) \int_{s'}\mathbf{r}(s,a,s')P(s'|s,a)ds'dsda - \int_{s,a} \rho(s, a) \ln_{q}\left(\frac{\rho(s, a)}{\int_{a'}\rho(s, a')da'}\right)dsda \\
	& \text{subject to}
	& &\forall s,\;a,\; \rho(s, a)\ge 0,\; \int_{a}\rho(s, a) da= d(s) + \int_{s',a'} P(s|s',a')\rho(s', a')ds'da',
\end{aligned}
\end{eqnarray}
where $P$ is a density function of the transition probability.
The optimization variable is a function $\rho$ and constraints imply the set of functions which satisfy the Bellman flow constraints.
Now, our objective function is the functional whose input is a function $\rho$ and output is the sum of expected rewards.
To analyze the continuous optimization problem,
we can employ generalized KKT condition for the functional optimization \cite{luenberger1997optimization}
where it mainly utilizes the functional derivatives to obtain the KKT conditions.
We can derive similar theoretical results by using the functional derivatives and generalized KKT conditions.

\section{Additional Examples}
\subsection{Full Experimental Results}
The entire results are shown in Figure \ref{fig:mujoco_tac} and \ref{fig:mujoco_all}.
For Hopper-v2,
TAC methods with $1 \leq q \leq 2$ show similar performance as shown in Figure \ref{fig:hopper_tac}.
Furthermore, compared to TD3 and SAC, TAC  also has similar performance as shown in Figure \ref{fig:hopper_all}.
These results suggest that Hopper-v2 may be easy to solve using TAC, SAC, and TD3.
For Swimmer-v2, 
TAC is stuck in local optima for every $q$ values.
However, TD3 and SAC also show poor performance while PPO, TRPO, and DDPG has better performance.
Since TAC is designed based on SAC and TD3, 
the performance of TAC seems to rely on the performance of SAC and TD3.
We believe that this problem may not occurs if we apply the Tsallis entropy maximization to other SG entropy based methods.
We also conduct the effect of the network capacity on on-policy methods: PPO and TRPO, for fair comparison, as shown in Figure \ref{fig:mujoco_onpol}.
We can realize that the large network capacity does not help the performance of PPO and TRPO.
Therefore. this result justifies the experiments on PPO and TRPO with smaller networks in our comparison. 

\begin{figure*}[t!]
\vspace{-10pt}
\centering
\subfigure[Hopper-v2]{\includegraphics[width=0.32\textwidth]{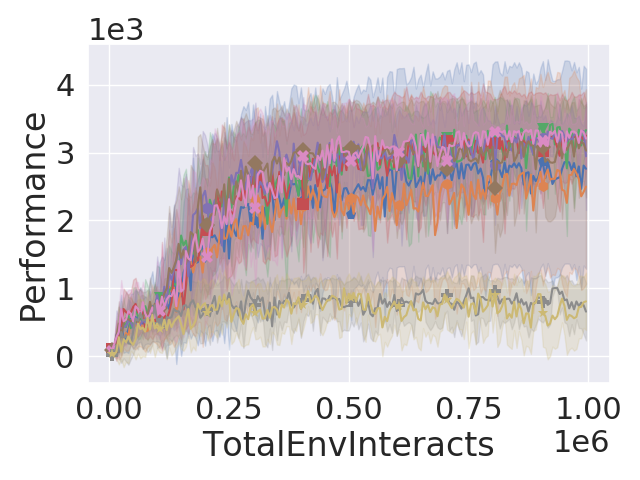}\label{fig:hopper_tac}}
\subfigure[Swimmer-v2]{\includegraphics[width=0.32\textwidth]{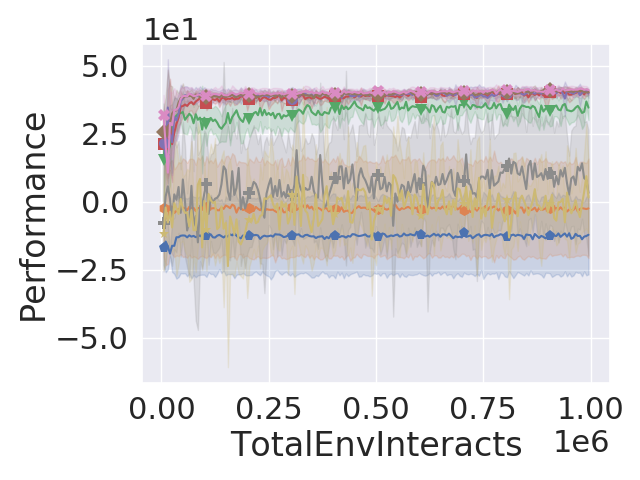}}
\subfigure[HalfCheetah-v2]{\includegraphics[width=0.32\textwidth]{fig/half_tac}}
\subfigure[Ant-v2]{\includegraphics[width=0.32\textwidth]{fig/ant_tac}}
\subfigure[Pusher-v2]{\includegraphics[width=0.32\textwidth]{fig/push_tac}}
\subfigure[Humanoid-v2]{\includegraphics[width=0.32\textwidth]{fig/human_tac}}
\vspace{-5pt}
\caption{Average training returns on MuJoCo environments. Tsallis actor-critics with $q$ values : $1.0,1.2,1,5$ and $1.7$ generally show better performance than other entropic indices. Real line is an average return over ten trials and shade area shows a variance.  }\label{fig:mujoco_tac}
\vspace{-5pt}
\end{figure*}

\begin{figure*}[t!]
\vspace{-5pt}
\centering
\subfigure[Hopper-v2]{\includegraphics[width=0.32\textwidth]{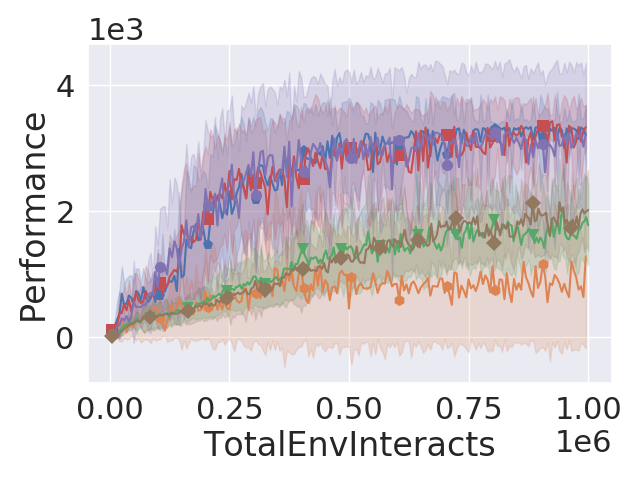}\label{fig:hopper_all}}
\subfigure[Swimmer-v2]{\includegraphics[width=0.32\textwidth]{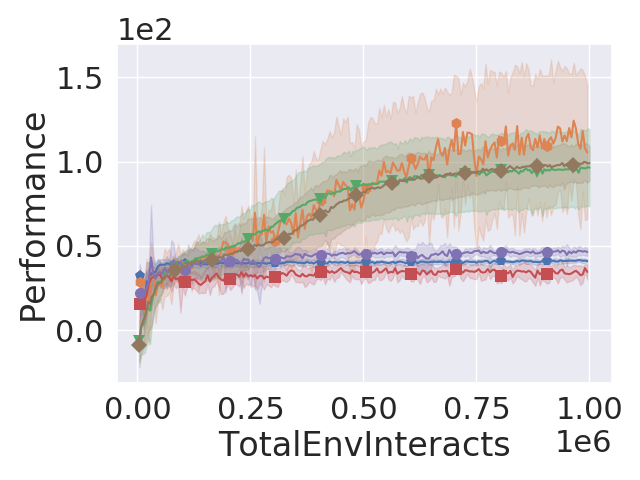}}
\subfigure[HalfCheetah-v2]{\includegraphics[width=0.32\textwidth]{fig/half_all}}
\subfigure[Ant-v2]{\includegraphics[width=0.32\textwidth]{fig/ant_all}}
\subfigure[Pusher-v2]{\includegraphics[width=0.32\textwidth]{fig/push_all}}
\subfigure[Humanoid-v2]{\includegraphics[width=0.32\textwidth]{fig/human_all}}
\vspace{-5pt}
\caption{Comparison to existing actor-critic methods on four MuJoCo environments. Soft actor-critic (red square line) is the same as Tsallis actorc-critic with $q=1$ and TAC (blue pentagon line) indicates Tasllis actor-critic with $q\neq 1$.  }\label{fig:mujoco_all}
\vspace{-5pt}
\end{figure*}

\begin{figure*}[t!]
\vspace{-10pt}
\centering
\subfigure[Hopper-v2]{\includegraphics[width=0.32\textwidth]{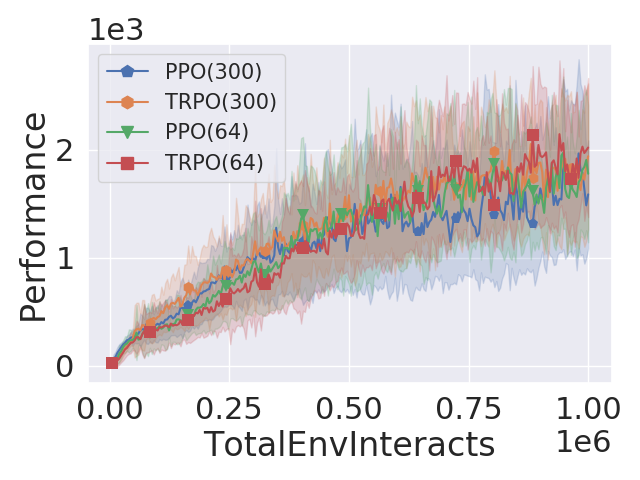}}
\subfigure[Swimmer-v2]{\includegraphics[width=0.32\textwidth]{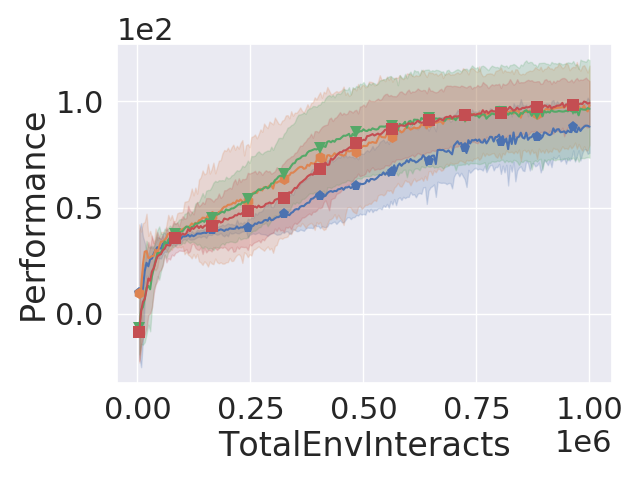}}
\subfigure[HalfCheetah-v2]{\includegraphics[width=0.32\textwidth]{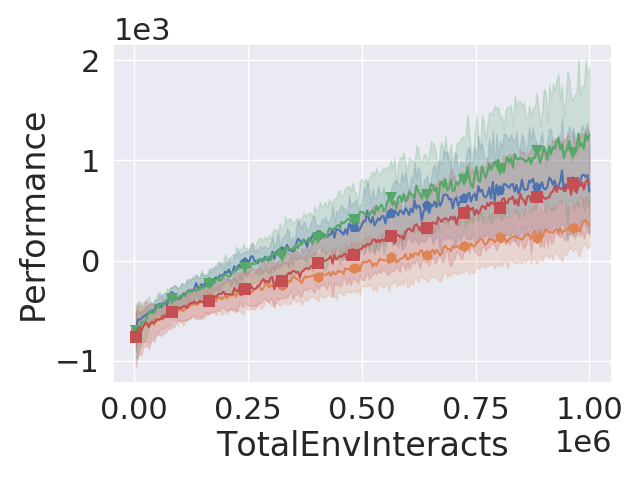}}
\subfigure[Ant-v2]{\includegraphics[width=0.32\textwidth]{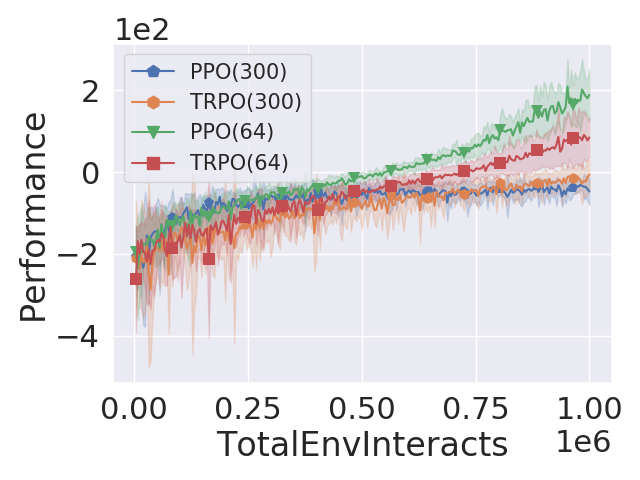}}
\subfigure[Pusher-v2]{\includegraphics[width=0.32\textwidth]{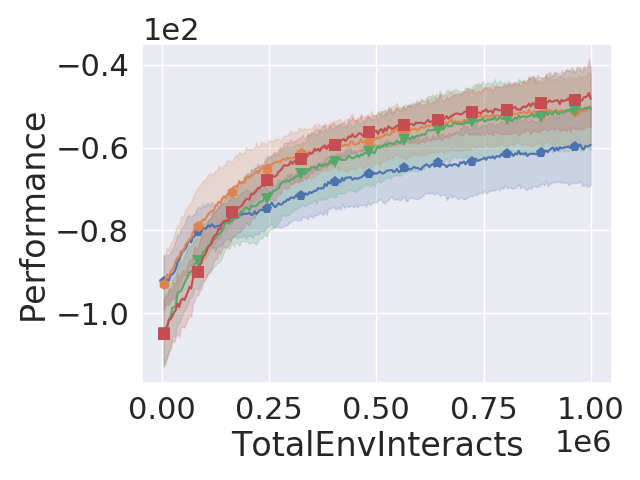}}
\subfigure[Humanoid-v2]{\includegraphics[width=0.32\textwidth]{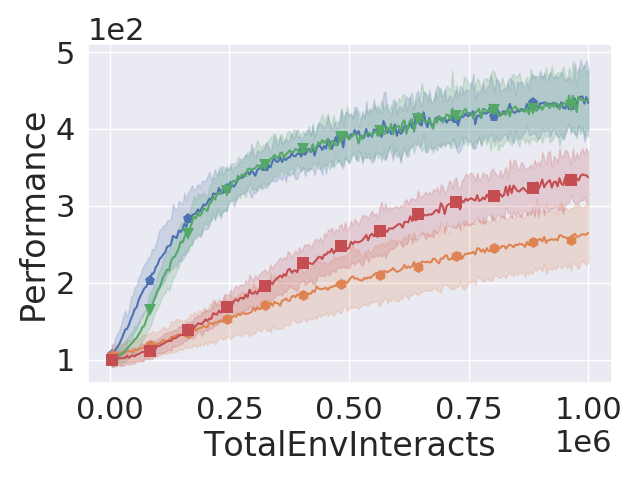}}
\vspace{-5pt}
\caption{Average training returns of PPO and TRPO on MuJoCo environments. The number in parentheses indicates the number of hidden units. Both algorithms have better performance when using smaller networks}\label{fig:mujoco_onpol}
\vspace{-5pt}
\end{figure*}

\subsection{Effect of $\alpha$}
We compare the effect of $\alpha$ and effect of $q$ values as shown in Figure \ref{fig:q_alpha}.
We compute the $\pi^{\star}_{q}$ on an MAB problem with the reward function shown in Figure \ref{fig:q_alpha}, while changing $\alpha$ and $q$ values.
We can see that all policy with different $q$ values converge to greedy policy when $\alpha\rightarrow0$.
However, the tendency of convergence is different depending on $q$ values.
First, the supporting set of $\pi^{\star}_{q}$ with entropy coefficient $\alpha$ is defined by 
\begin{equation*}
1+(q-1)\left(\frac{\mathbf{r}(a)}{\alpha q}- \psi_{q}\left(\frac{\mathbf{r}}{\alpha q}\right)\right)  > 0.
\end{equation*}
Since we only consider $q\geq 1$ and $(q-1)$ is positive, the supporting set becomes bigger as $\alpha$ goes to infinity.
Thus, for fixed $q$, large $\alpha$ makes $\pi_{q}^{\star}$ more uniform but, from the supporting set condition,
the probability mass are distributed over restrict elements.
For example, when $q=2.0$, 
the probability mass are only distributed on four elements while $\alpha$ vareis from $0.1$ to $2.0$.
On the contrary, when $q=1.0$
the probability mass are distributed over the entire action space while $\alpha$ vareis from $0.1$ to $2.0$.
We would like to note that $q$ value controls the supporting set and $\alpha$ controls the distribution over the supporting set.

\subsection{Example of Bounds for $q$-Maximum}
From theorem \ref{thm:bound}, 
we have the bounds for $q$-maximum as follows,
\begin{eqnarray*}
\begin{aligned}
 \max_{x}(f(x))\leq \mathop{q\textnormal{-max}}_{x}\left(f(x)\right) \leq \max_{x}(f(x)) - \ln_{q} \left(1/|\mathcal{X}|\right)
\end{aligned}
\end{eqnarray*}
In example, we set $\mathcal{X}=\{0,1\}$ and $f(x)$ is deinfed as $f(0) = 0, f(1) = c$.
We see the tendecy of $q$-maximum when $c$ varies from $-2$ to $2$.
Then, $\max_{x}(f(x))$ becomes $\max([c,0])$ and 
we compute the $q\text{-max}([c,0])$ using numerical solver.
Since $\mathcal{X}$ has two elements,
the upper bound is $\max([c,0])-\ln_{q}(1/2)$.

Examples of $q$-maximum with different $q$ values are shown in Figure \ref{fig:qmx_bnd}.
It can be observed that, as $q$ increases, the bounds become tighter.
Note that the gap becomes largest when $q=1$.
This gap sometimes leads to overestimation error when we use $q$-maximum to compute the target value of value networks.

\begin{figure*}[t!]
\includegraphics[width=\textwidth]{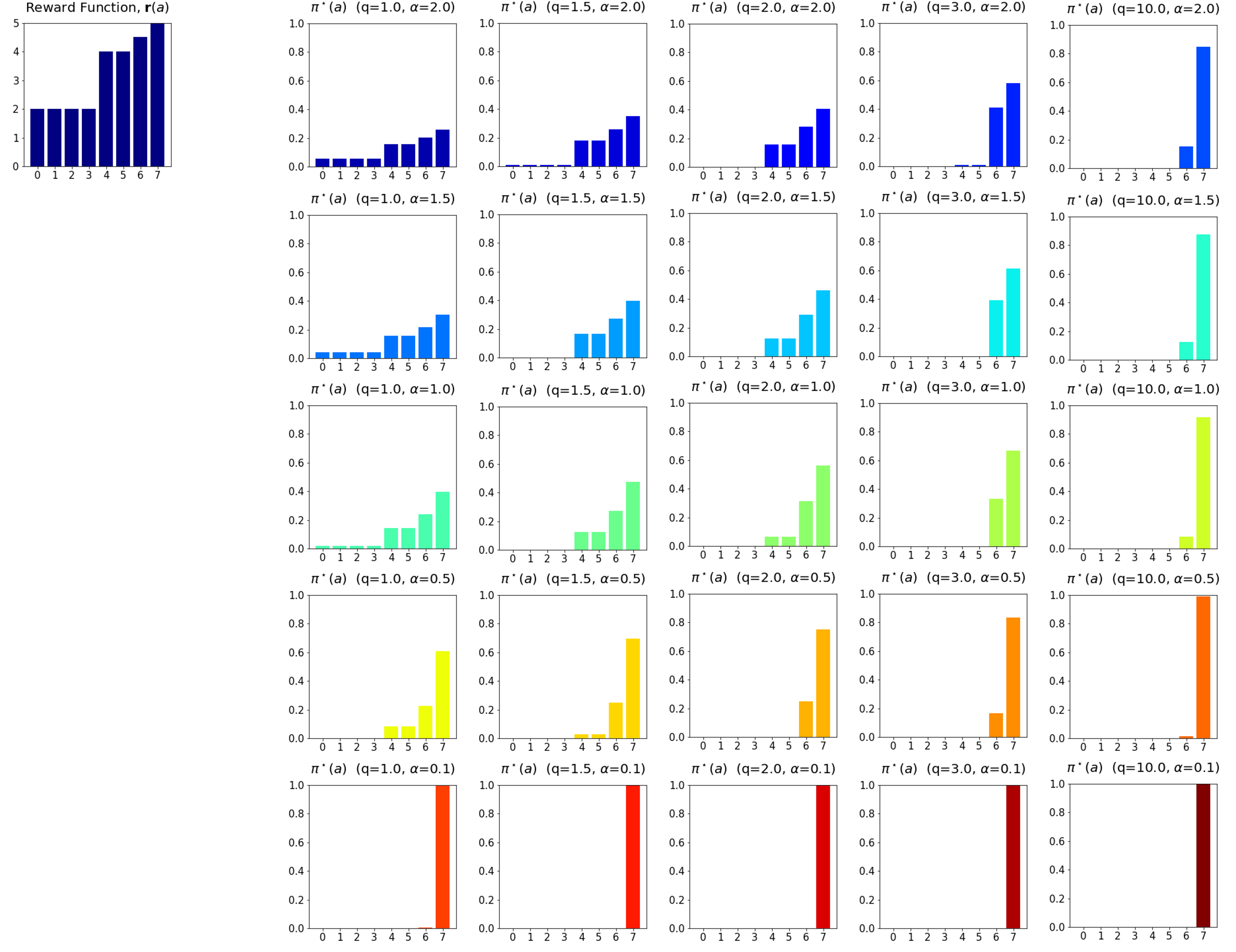}
\centering
\caption{Example of $\pi_{q}^{\star}(a)$ with various coefficients of entropy $\alpha$ varying from $2.0$ to $0.1$  and \textit{entropic indices} varying from $1.0$ to $10.0$, respectively. } \label{fig:q_alpha}
\end{figure*}

\begin{figure*}[t!]
\includegraphics[width=\textwidth]{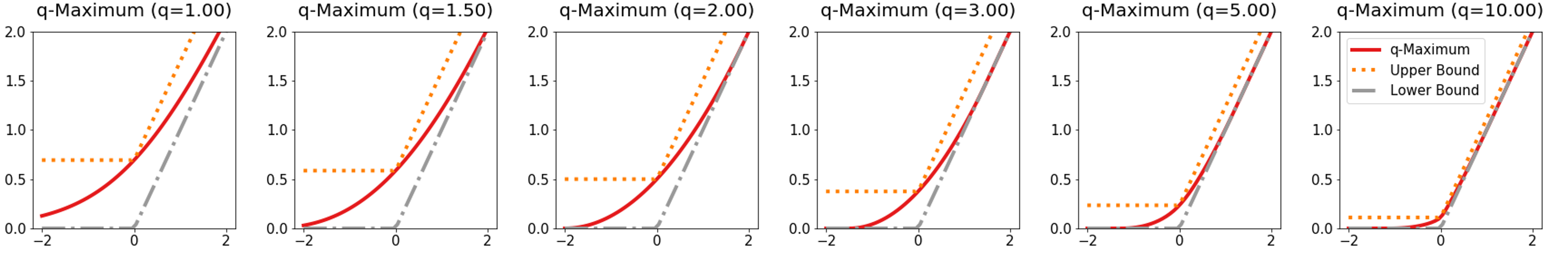}
\centering
\caption{Example of $q$-maximum operator with different $q$ values. The figures show $q\textnormal{-max}([c,0])$ over $c\in[-2,2]$. The bounds are computed by Theorem \ref{thm:bound}} \label{fig:qmx_bnd}
\end{figure*}

\section{Implementation Detail}

Our algorithm maintains five networks to model a policy $\pi_{\phi}$, state value $V_{\psi}$, target state value $V_{\psi^{-}}$, two state action values $Q_{\theta_{1}}$ and $Q_{\theta_{2}}$.
We also utilize a replay buffer $\mathcal{D}$ where every interaction data $(s_{t},a_{t},r_{t+1},s_{t+1})$ is stored in a replay buffer and it is sampled when updating the networks.
Value networks $V_{\psi}, Q_{\theta_{1}}$ and $Q_{\theta_{2}}$ are updated using (\ref{def:tbe_op_q}) and $\pi_{\phi}$ is updated using (\ref{def:tp_imp}).

To update state value network $V_{\psi}$,
we minimize the following loss,
\begin{eqnarray}\label{eqn:v_obj}
\small
\begin{aligned}
J_{\psi} = \mathop{\mathbbm{E}}_{s_t, a_t \sim \mathcal{B}}\left[\frac{1}{2}(y_{t} - V_{\psi}(s_{t}))^{2}\right]
\end{aligned}
\end{eqnarray}
where $\mathcal{B}\subset\mathcal{D}$ is a mini-batch and $y_{t} = Q_{\min}(s_{t},a_{t}) - \alpha \ln_{q}(\pi_{\phi}(a_{t}|s_{t})$ where $Q_{\min}(s_{t},a_{t}) = \min [ Q_{\theta_{1}}(s_{t},a_{t}), Q_{\theta_{2}}(s_{t},a_{t})]$.
Using two approximations for $Q^{\pi}$ is followed from \cite{fujimoto2018td3} which is known to prevent overestimation problem in Q learning.
To update both $\theta_{1}$ and $\theta_{2}$, we minimize the following update rule, 
\begin{eqnarray}\label{eqn:q_obj}
\small
\begin{aligned}
J_{\theta} = \mathop{\mathbbm{E}}_{b_{t} \sim \mathcal{B}}\left[ \frac{1}{2}(Q_{\theta}(s_{t},a_{t}) - r_{t+1} - \gamma V_{\psi^{-}}(s_{t+1}))^{2}\right],
\end{aligned}
\end{eqnarray}
where $b_{t}$ is $(s_t, a_t, s_{t+1}, r_{t+1})$.
These update rules come from the Tsallis policy evaluation step.
$\psi^{-}$ is updated by an exponential moving average with ratio $\tau$, i.e., $\psi^{-}\leftarrow(1-\tau)\psi^{-}+\tau\psi$.

To update an actor network,
we minimize a policy improvement objective defined as
\begin{eqnarray}\label{eqn:pol_obj}
\small
\begin{aligned} 
J_{\phi} = \mathop{\mathbbm{E}}_{s_{t} \sim \mathcal{B}}\left[ \mathop{\mathbbm{E}}_{a \sim \pi_{\phi}} \left[\alpha \ln_{q}(\pi_{\phi}(a|s_{t})) - Q_{\theta}(s_{t},a)\right]\right].
\end{aligned}
\end{eqnarray}
Note that $a$ is sampled from $\pi_{\phi}$ not a replay buffer.
Since updating $J_{\phi}$ requires to compute a stochastic gradient,
we use a reparameterization trick similar to \citet{haarnoja2018sac} instead of a score function estimation.
In our implementation, a policy function is defined as a reparameterizable distribution such as Gaussian distribution
where we can reparameterize the action sampled from the actor network, $a = f_{\phi}(s,\epsilon)$ where $\epsilon$ is often modeled by a well-known distribution $P_{\epsilon}$ which can be easily sampled.
Consequently, we can rewrite $J_{\phi}$ with a reparameterized action and a stochastic gradient  is computed as 
$$\nabla_{\phi} J_{\phi} = \mathop{\mathbbm{E}}_{s_{t} \sim \mathcal{B}}\left[\mathop{\mathbbm{E}}_{\epsilon \sim P_{\epsilon}} \left[\alpha \nabla_{\phi}\ln_{q}(\pi_{\phi}(a|s_{t})) - \nabla_{\phi}Q_{\theta}(s_{t},a)\right]\right],$$
where $a=f_{\phi}(s_{t},\epsilon)$.
\begin{algorithm}[t!]
\caption{Tsallis Actor Critic (TAC)}
\begin{algorithmic}[1] \label{alg:tac}\footnotesize
\STATE {\bfseries Input:} Total time steps $t_{\max}$, Max episode length $l_{\max}$, Memory size $N$, Entropy coefficient $\alpha$,  Entropic index $q$, Moving average ratio $\tau$, Environment \textit{env}
\STATE {\bfseries Initialize:} $\psi, \psi^{-}, \theta_{1}, \theta_{2}, \phi$, $\mathcal{D}$ : Queue with size $N$, $t=0$, $t_{e}=0$
\WHILE{$t\leq t_{\max}$}
\STATE $a_{t}\sim\pi_{\phi}$ and $\mathbf{r}_{t+1}, s_{t+1}, d_{t+1} \sim env$ where $d_{t+1}$ is a terminal signal.
\STATE $\mathcal{D}\leftarrow\mathcal{D}\cup\{(s_{t},a_{t},\mathbf{r}_{t+1},s_{t+1},d_{t+1})\}$
\STATE $t_{e}=t_{e}+1$, $t =t+1$
\IF{$d_{t+1}=$ True or $t_{e}=l_{\max}$}
\FOR{$i=1$ to $t_{e}$}
\STATE Randomly sample a mini-batch $\mathcal{B}$ from $\mathcal{D}$
\STATE Minimize $J_{\psi}, J_{\theta_{1}}, J_{\theta_{2}}$, and $J_{\phi}$ using a stochastic gradient descent
\STATE Update $J_{\psi^{-}}$ with $\tau$
\ENDFOR
\STATE Reset $env$,  $t_{e}=0$
\ENDIF
\ENDWHILE
\end{algorithmic}
\end{algorithm}

\subsection{Reparameterization Trick}
We follows the reparameterization trick used in the soft actor-critic method \cite{haarnoja2018sac}.
For continuous random variable, the policy network often model the Gaussian distribution where the output of $\pi_{\phi}$ is the mean $\mu_{\phi}(s)$ and standard deviation $\sigma_{\phi}(s)$ of Gaussian distribution.
However, in most continuous control problems,
the action space is often bounded.
In this regards,
we apply a tangent hyperbolic (tanh) to the Gaussian samples 
$$
a = f_{\phi}(s,\epsilon) = \text{tanh}(z)
$$
where

$$
z=\mu_{\phi}(s)+\epsilon \sigma_{\phi}(s)
$$
where $\epsilon \sim \mathcal{N}(0,\mathbf{I})$.
Then, the likelihood of actions $\pi_{\phi}(a|s)$
is computed as
$$
\pi(a|s) = \mathcal{N}\left(z ; \mu_{\phi}(s),\sigma_{\phi}(s)\right)\left|\frac{\mathrm{d}a}{\mathrm{d}z}\right|^{-1}
$$
where
$$
\left|\frac{\mathrm{d}a}{\mathrm{d}z}\right|^{-1} = \Pi_{i=1}^{D} (1-\text{tanh}(z_{i}))^{-1}
$$
where $D$ is a dimension of $z$ or $a$ and $z_{i}$ is the $i$th element of $z$.
Finally, the $q$-logarithm of the policy distribution is 
$$
\ln_{q}\left(\mathcal{N}\left(z ; \mu_{\phi}(s),\sigma_{\phi}(s)\right)\left|\frac{\mathrm{d}a}{\mathrm{d}z}\right|^{-1}\right)
$$
\subsection{Numerical Issue}
Since we handle the continuous action space, the policy is modeled as a density function of a continuous random variable.
Unlikely to a probability mass function, a probability density function (pdf) sometimes diverges to infinity (or the maximum number of a computing machine)
when the probability is concentrated at a single point.
In this  case, the large pdf value induces a large gradient which makes the learning procedure unstable.
Thus, in our implementation, the pdf value is restricted to no greater than $10^{\frac{8}{q-1}}$.
This numerical issue is often caused when $q\geq2$.
For $q\geq2$, we use the following likelihood.
$$
\ln_{q}\left(\min\left(10^{\frac{8}{q-1}},\mathcal{N}\left(z ; \mu_{\phi}(s),\sigma_{\phi}(s)\right)\left|\frac{\mathrm{d}a}{\mathrm{d}z}\right|^{-1}\right)\right)\;\;\left(\leq \frac{10^{8}-1}{q-1}\right)
$$

\subsection{Hyperparameter Settings}
{\centering
\begin{tabular}{c|c}
Parameter                      &       Value    \\
\toprule
\toprule
Optimizer                      &      Adam in TensorFlow     \\
Learning rate                  &   $1e^{-3}$      \\
Discount factor                &  $0.99$         \\
Replay buffer size             &   $1e^{6}$      \\
Number of Minimum samples in buffer &  $1e^{5}$   \\
Number of Hidden Layers        & $1$   \\
Number of Hidden units         &    $300$         \\
Activation function            &   ReLu      \\
Number of samples in minibatch &   $100$   \\
Moving average ratio           &  $0.005$    \\
Seeds                          & 0 10 20 30 40 50 60 70 80 90\\ \bottomrule
\end{tabular}\par}

{\centering
\begin{tabular}{c|c|c}
Environment                     &       Entropy Coefficient, $\alpha$ & (Best) Entropic Index, $q$ \\
\toprule
\toprule
Hopper-v2                   & $0.2$ &  $2.0$ (slightly better) \\
Swimmer-v2                   & $0.2$ & $2.0$ (slightly better) \\
HalfCheetah-v2              & $0.2$ &  $1.5$   \\
Ant-v2          & $0.2$ &  $1.5$ \\
Pusher-v2&$0.2$ & $1.2$ (slightly better)  \\
Humanoid-v2 &$0.05$&$1.2$ \\ \bottomrule
\end{tabular}\par}

\bibliographystyle{IEEEtran}
\bibliography{bib_tac}
\end{document}